\documentclass{article} % For LaTe\X2e
\usepackage{nips14submit_e,times}
\usepackage{graphicx} % more modern

\usepackage[T1]{fontenc}
\usepackage[utf8]{inputenc}

\usepackage{algorithm}
\usepackage{algpseudocode}
\algrenewcommand\algorithmicindent{0.9em}
\usepackage{subfigure}
\usepackage{epsf}
\usepackage{amsmath}
\usepackage{amssymb}
\usepackage{amsfonts}
\usepackage{amsthm}
\usepackage{dsfont}
\usepackage{enumitem}
\usepackage{float}
\usepackage{cases}
\usepackage{verbatim}
\usepackage{url}
\usepackage{tikz}
\usepackage{wrapfig}
\usetikzlibrary{matrix}
\usepackage{todonotes}
\usepackage[font=small]{caption}

\usepackage[numbers]{natbib}

\DeclareMathOperator*{\argmin}{arg\,min} 

\newlength{\minipagewidth}
\newlength{\minipagewidthx}
\newcommand{\bookboxx}[1]{\small
\par\medskip\noindent
\framebox[\columnwidth]{
\begin{minipage}{1.05\dimexpr\textwidth-\parindent\relax} {#1} \end{minipage} } \par\medskip }

\newcommand{\beq}{\begin{equation}}
\newcommand{\eeq}{\end{equation}}

\newcommand{\beqa}{\begin{eqnarray}}
\newcommand{\eeqa}{\end{eqnarray}}

\newcommand{\beqan}{\begin{eqnarray*}}
\newcommand{\eeqan}{\end{eqnarray*}}

\renewcommand{\S}{{\cal S}}
\newcommand{\hS}{\widehat{\cal S}}
\newcommand{\hX}{\widehat{\cal X}}
\newcommand{\hY}{\widehat{\cal Y}}
\newcommand{\D}{\mathcal{D}}

\title{Best-Arm Identification in Linear Bandits}

\author{
Marta Soare ~~~~~~~~~ Alessandro Lazaric ~~~~~~~~~ R\'emi Munos\thanks{\small{This work was done when the author was a visiting researcher at Microsoft Research New-England.}}~~\thanks{\small{Current affiliation: Google DeepMind.}} \\
INRIA Lille -- Nord Europe, SequeL Team\\
\texttt{\{marta.soare,alessandro.lazaric,remi.munos\}@inria.fr}
}

\newcommand{\C}{\mathcal C}

\newcommand{\X}{\mathcal X}
\newcommand{\Y}{\mathcal Y}

\newcommand{\Prob}{\mathbb P}

\newcommand{\bx}{\mathbf{x}}

\newcommand{\htheta}{\hat\theta}

\newcommand{\hx}{\hat x}

\newcommand{\eps}{\varepsilon}

\renewcommand{\Re}{\mathbb{R}}

\newcommand{\hDelta}{\widehat{\Delta}}

\newtheorem{lemma}{Lemma}

\newtheorem{proposition}{Proposition}

\newtheorem{theorem}{Theorem}

\nipsfinalcopy % Uncomment for camera-ready version
\begin{document}

\maketitle

\vspace{-0.3in}

%\begin{abstract}
%We study the best-arm identification problem in the linear bandit setting, where the rewards of the arms depend linearly on an unknown parameter $\theta^*$ and the objective is to return the arm with the largest reward.%. More precisely, given a finite input set $\X$, we want to identify $\argmax_{x \in \X} x^\top \theta^*$.% and, as opposed to the stochastic multi-armed bandit setting, the pull to an arm gives information about the expected rewards of all arms. 
%We characterize the complexity of the problem and derive sample allocation strategies which efficiently exploit the geometry of the input space to identify the best arm with a fixed confidence. In particular, we show how the global linear structure can be exploited to improve the estimate of the reward of near-optimal arms. %while minimizing the sampling budget. 
%We analyze the proposed strategies and compare their empirical performance. Finally, as a by-product of our analysis, we point out the connection of the problem to the $G$-optimality criterion and offer the novel characterization of an optimal solution that depends both on the local properties around the optimal values and on the global geometry of the input space.
%\end{abstract}

\begin{abstract}
\vspace{-0.1in}
We study the best-arm identification problem in linear bandit, where the rewards of the arms depend linearly on an unknown parameter $\theta^*$ and the objective is to return the arm with the largest reward. We characterize the complexity of the problem and introduce sample allocation strategies that pull arms to identify the best arm with a fixed confidence, while minimizing the sample budget. In particular, we show the importance of exploiting the global linear structure to improve the estimate of the reward of near-optimal arms. We analyze the proposed strategies and compare their empirical performance. Finally, as a by-product of our analysis, we point out the connection 
to the $G$-optimality criterion used in optimal experimental design.
\end{abstract}

%%%%%%%%%%%%%%%%%%%%%%%%%%%%%%%%%%%%%%%%%%%%%%%%%%%%%%%%%%%%%%%%%%%%%%%%%%%%%%%
%% INTRODUCTION 
%%%%%%%%%%%%%%%%%%%%%%%%%%%%%%%%%%%%%%%%%%%%%%%%%%%%%%%%%%%%%%%%%%%%%%%%%%%%%%%

\vspace{-0.2in}
\section{Introduction}\label{s:intro}
\vspace{-0.1in}

The stochastic multi-armed bandit problem (MAB)~\citep{Robbins52} offers a simple formalization for the study of sequential design of experiments. In the standard model, a learner sequentially chooses an arm out of $K$ and receives a reward drawn from a fixed, unknown distribution relative to the chosen arm. While most of the literature in bandit theory focused on the problem of maximization of cumulative rewards, where the learner needs to trade-off exploration and exploitation, recently the \emph{pure exploration} setting~\citep{pure-exploration} has gained a lot of attention. Here, the learner uses the available budget to identify as accurately as possible the best arm, without trying to maximize the sum of rewards. Although many results are by now available in a wide range of settings (e.g., best-arm identification with fixed budget~\citep{audibert2010best,jamieson2013lil-ucb} and fixed confidence~\citep{action-elimination}, subset selection~\citep{bubeck2013multiple,kaufmann2013information}, and multi-bandit~\citep{gabillon2011multi-bandit}), most of the work considered only the multi-armed setting, with $K$ independent arms.
%
%Typically,  given a limited number of pulls, the goal of the learned is to maximize the sum of rewards, given rise to the so-called \emph{exploration-exploitation trade-off}, where exploration of the arms improves the knowledge on the environment and the exploitation of estimated leads to maximize the sum of rewards. More recently, a different viewpoint on the same problem has received considerable attention. In the \emph{pure exploration} setting~\citep{pure-exploration}, the forecaster uses the available budget to identify as accurately as possible the best arm, without trying to maximize the sum of rewards. Although many results are by now available in a wide number of settings (e.g., best-arm identification with fixed budget~\citep{audibert2010best,jamieson2013lil-ucb} and fixed confidence~\citep{action-elimination,gabillon2012best}, subset selection~\citep{bubeck2013multiple,kaufmann2013information}, and multi-bandit~\citep{gabillon2011multi-bandit}), most of the work focused on the multi-armed setting, with $K$ independent arms.
%

An interesting variant of the MAB setup is the stochastic \emph{linear bandit} problem (LB), introduced in~\citep{Auer02}. In the LB setting, the input space $\X$ is a subset of $\Re^d$ and when pulling an arm $x $, the learner observes a reward whose expected value is a linear combination of $x$ %the chosen arm 
and an unknown parameter $\theta^*\in\Re^d$. %Since the arms are correlated and the uncertainty of the environment is concentrated in the unknown parameter $\theta^* \in \R^d$ characterizing the linear function, this particular setting becomes more natural in applications where the number of available actions (or arms) is very large. In fact, 
Due to the linear structure of the problem, pulling an arm gives information about the parameter $\theta^*$ and indirectly, about the value of other arms. Therefore, the estimation of $K$ mean-rewards is replaced by the estimation of the $d$ features of $\theta^*$. %, which can be particularly handy in applications such as  ??
While in the exploration-exploitation setting the LB has been widely studied both in theory and in practice (e.g.,~\citep{AbPaSz11,LinUCB}),  in this paper we focus on the pure-exploration scenario.

The fundamental difference between the MAB and the LB best-arm identification strategies stems from the fact that %becomes straightforward when we compare the sampling strategies stemming from the two settings. While in the multi-armed bandit setup, 
in MAB an arm is no longer pulled as soon as its sub-optimality is evident (in high probability), %(since the discarding is equivalent to its sub-optimality)
while in the LB setting even a sub-optimal arm may offer valuable information about the parameter vector $\theta^*$ and thus improve the accuracy of the estimation in discriminating among near-optimal arms. For instance, consider the situation when $K\!-\!2$ out of $K$ arms are already discarded. In order to identify the best arm, MAB algorithms would concentrate the sampling on the two remaining arms to increase the accuracy of the estimate of their mean-rewards until the discarding condition is met for one of them. On the contrary, a LB pure-exploration strategy would seek to pull the arm $x \in \X$ whose observed reward allows to refine the estimate $\theta^*$ along the dimensions which are more suited in discriminating between the two remaining arms. %Depending on the geometry of the input space, this \emph{most informative} might as well be one of the $k-2$ previously discarded arms. 
Recently, the best-arm identification in linear bandits has been studied in a fixed budget setting~\citep{HoffmanSF14}, in this paper we study the sample complexity required to identify the best-linear arm with a fixed confidence.

\section{Preliminaries}\label{s:prelim}
\vspace{-0.1in}

\textbf{The setting.} We consider the standard linear bandit model. Let $\X\subseteq \Re^d$ be a finite set of arms, where $|\X| = K$ %< \infty$ 
and the $\ell_2$-norm of any arm $x\in\X$, denoted by $||x||$, is upper-bounded by $L$. Given an unknown parameter $\theta^*\in\Re^d$, we assume that each time an arm $x\in\X$ is pulled, a random reward $r(x)$ is generated according to the linear model $r(x) = x^\top \theta^* + \eps$,
%
%\begin{align}\label{eq:linear.model}
%r(x) = x^\top \theta^* + \eps,
%\end{align}
%
where $\eps$ %\sim \mathcal{N}(0,\sigma^2). $ 
is a zero-mean i.i.d. noise bounded in $[-\sigma;\sigma]$.
Arms are evaluated according to their expected reward $x^\top \theta^*$ and we denote by $x^* = \arg\max_{x\in\X} x^\top \theta^*$ the best arm in $\X$. Also, we use $\Pi(\theta) = \arg\max_{x\in\X} x^\top \theta$ to refer to the best arm corresponding to an arbitrary parameter $\theta$. Let $\Delta(x,x') = (x-x')^\top\theta^*$ be the value \textit{gap} between two arms, then we denote by $\Delta(x) = \Delta(x^*,x)$ the gap of $x$ w.r.t.\ the optimal arm and by $\Delta_{\min} = \min_{x\in\X} \Delta(x)$ the minimum gap, where $\Delta_{\min}> 0$. We also introduce the sets $\Y = \{y=x-x', \forall x,x'\in\X\}$ and $\Y^*=\{y=x^*-x, \forall x\in\X\}$ containing all the directions obtained as the difference of two arms (or an arm and the optimal arm) and we redefine accordingly the gap of a direction as $\Delta(y) = \Delta(x,x')$ whenever $y=x-x'$.

\textbf{The problem.} We study the best-arm identification problem. Let $\hx(n)$ be the estimated best arm returned by a bandit algorithm after $n$ steps. We evaluate the \textit{quality} of $\hx(n)$ by the simple regret $R_n = (x^* - \hx(n))^\top \theta^*$.
%%
%\begin{align}\label{eq:simple.regret}
%R_n = (x^* - \hx(n))^\top \theta^*.
%\end{align}
%%
While different settings can be defined (see~\citep{gabillon2012best} for an overview), here we focus on the $(\epsilon,\delta)$-best-arm identification problem (the so-called PAC setting), where given $\epsilon$ and $\delta\in(0,1)$, the objective is to design an allocation strategy and a stopping criterion so that when the algorithm stops, the returned arm $\hx(n)$ is such that $\Prob\big( R_n \geq \epsilon \big) \leq \delta$,
%
%\begin{align}\label{eq:fixed.confidence}
%\Prob\big( R_n \geq \epsilon \big) \leq \delta,
%\end{align}
%
while minimizing the needed number of steps. More specifically, we will focus on the case of $\epsilon=0$ and we will provide high-probability bounds on the sample complexity $n$.

\textbf{The multi-armed bandit case.} %The best-arm identification problem has been studied in the MAB in multiple settings \citep{audibert2010best,action-elimination,jamieson2013lil-ucb}. 
In MAB, the complexity of best-arm identification is characterized by the gaps between arm values, following the intuition that the more similar the arms, the more pulls are needed to distinguish between them. More formally, the complexity is given by the problem-dependent quantity $H_\text{MAB} = \sum_{i=1}^K \frac{1}{\Delta_i^2}$ %with $\Delta_i > 0$ 
i.e.,
%
%\begin{equation}\label{eq:mab.complexity}
%H_\text{MAB} = \sum_{i=1}^K \frac{1}{\Delta_i^2},
%\end{equation}
%
the inverse of the pairwise gaps between the best arm and the suboptimal arms. In the fixed budget case, $H_\text{MAB}$ determines the probability of returning the wrong arm~\citep{audibert2010best}, while in the fixed confidence case, it characterizes the sample complexity~\citep{action-elimination}.

%In particular, \citep{action-elimination} study the minimal sampling time required to identify the best arm in the fixed-confidence case. They propose a phase-algorithm, where at each of the $K-1$ phases all potentially optimal arms are sampled uniformly, then the arm with the lowest empirical value is discarded at the end of each phase. Overall, the proposed algorithm has a sample complexity of order $\mathcal O \left(\sum_{i=1}^k \frac{1}{\Delta_i^2}\log(\Delta_i k/\delta )\right)$ corresponding to the number of pulls needed to identify the best arm with probability of at least $1-\delta.$ 
%When the expected values of the arms are known (but the matching of the arms to the expected values in not) the lengths of the phases can be adapted such that with probability $1-\delta$ the worst arm is discarded using the smallest number of pulls to the arms, leading to a sample complexity of order $\mathcal O \left(\log(k/\delta) \sum_{i=1}^k \frac{1}{\Delta_i^2}\right)$. On the other hand, when the expected values are not given in advance, the algorithms need longer phases, which leads to 
%$\mathcal O \left(\sum_{i=1}^k \frac{1}{\Delta_i^2}\log(\Delta_i k/\delta )\right)$ pulls needed to identify the best arm with probability of at least $1-\delta.$ 

\textbf{Technical tools.} Unlike in the multi-arm bandit scenario where pulling one arm does not provide any information about other arms, in a linear model we can leverage the rewards observed over time %and exploit the linear structure 
to estimate the expected reward of all the arms in $\X$. Let $\bx_n = (x_1,\ldots,x_n)\in\X^n$ be a sequence of arms and $(r_1,\ldots,r_n)$ the corresponding observed (random) rewards. An unbiased estimate of $\theta^*$ can be obtained by ordinary least-squares (OLS) as $\htheta_n = A_{\bx_n}^{-1} b_{\bx_n}$,
%
%\begin{align}\label{eq:ols}
%\htheta_n = A_{\bx_n}^{-1} b_{\bx_n},
%\end{align}
% 
where $A_{\bx_n} = \sum_{t=1}^n x_t x_t^\top \in \Re^{d\times d}$ and $b_{\bx_n} = \sum_{t=1}^n x_t r_t \in \Re^d$. 
%The OLS estimate enjoys a series of interesting properties. First it is an unbiased estimator since $\E[\htheta_n|\bx_n] = \theta^*$. Furthermore, for any arm $x\in\X$ the mean-squared error of the estimated reward $x^\top\htheta_n$ is
%%
%\begin{align}\label{eq:mse.ols}
%\E\big[(x^\top \theta^* - x^\top \htheta_n)^2| \bx_n\big] =  \V[\eps]\; x^\top A_{\bx_n}^{-1} x = \V[\eps]\; ||x||_{A_{\bx_n}^{-1}}^2,
%\end{align}
%%
%where $\V[\eps]$ is the variance of the additive noise and $||x||_{M}$ denotes the $M$-weighted norm of $x$. Furthermore, we can also bound the prediction error for any fixed sequence $\bx_n$ in high-probability. 
For any fixed sequence $\bx_n$, through Azuma's inequality, the prediction error of the OLS estimate is upper-bounded in high-probability as follows. 

\begin{proposition}\label{p:bound}
Let $c = 2\sigma \sqrt{2}$ and $c'= 6/\pi^2$. %With probability $1-\delta, \forall n \in \mathbb{N}, \forall x\in\X$, and 
For every fixed sequence $\bx_n$, we have\footnote{Whenever Prop.1 is used for all directions $y\in\Y$, then the logarithmic term becomes $\log(c'n^2K^2/\delta)$ because of an additional union bound. For the sake of simplicity, in the sequel we always use $\log_n(K^2/\delta)$.}
\begin{align}\label{eq:err.bound}
\mathbb{P} \left( \forall n \in \mathbb{N}, \forall x\in\X,
\big| x^\top \theta^* - x^\top \htheta_n \big| \leq c ||x||_{A_{\bx_n}^{-1}} \sqrt{\log(c' n^2 K/\delta)}
\right) \geq 1-\delta.
\end{align}
\end{proposition}

While in the previous statement $\bx_n$ is fixed, a bandit algorithm adapts the allocation in response to the rewards observed over time. In this case a different high-probability bound is needed.

\begin{proposition}[Thm.~2 in~\citep{AbPaSz11}]\label{p:adaptive.bound}
Let $\htheta_n^\eta$ be the solution to the regularized least-squares problem with regularizer $\eta$ and let $\widetilde{A}_{\bx}^\eta=\eta I_d + A_{\bx}$. Then for all $x\in\X$ and every adaptive sequence $\bx_n$ such that at any step $t$, $x_t$ only depends on $(x_1,r_1,\ldots,x_{t-1},r_{t-1})$, w.p. $1-\delta$, we have
\begin{align}\label{eq:err.adaptive.bound}
\big| x^\top \theta^* - x^\top \htheta_n^{\eta} \big| \leq ||x||_{(\widetilde{A}^\eta_{\bx_n})^{-1}} \bigg(\sigma \sqrt{d\log \Big(\frac{1+nL^2/\eta}{\delta}\Big)} +\eta^{1/2}||\theta^*||\bigg).
\end{align}
%\todoa{The constant should be made explicit.}
\end{proposition}

The crucial difference w.r.t.\ Eq.~\ref{eq:err.bound} is an additional factor $\sqrt{d}$, the price to pay for adapting $\bx_n$ to the samples. %This difference will be critical in defining adaptive algorithms for best-arm identification. 
In the sequel we will often resort to the notion of design (or ``soft'' allocation) $\lambda \in \D^k$, which prescribes the \textit{proportions} of pulls to arm $x$ and $\D^k$ denotes the simplex $\X$. The counterpart of the design matrix $A$ for a design $\lambda$ is the matrix $\Lambda_\lambda = \sum_{x\in\X} \lambda(x) xx^\top$. From an allocation $\bx_n$ we can derive the corresponding design $\lambda_{\bx_n}$ as $\lambda_{\bx_n}(x) = T_n(x)/n$, where $T_n(x)$ is the number of times arm $x$ is selected in $\bx_n$, and the corresponding design matrix is $A_{\bx_n} = n\Lambda_{\lambda_{\bx_n}}$.

%%%%%%%%%%%%%%%%%%%%%%%%%%%%%%%%%%%%%%%%%%%%%%%%%%%%%%%%%%%%%%%%%%%%%%%%%%%%%%%
%% ORACLE
%%%%%%%%%%%%%%%%%%%%%%%%%%%%%%%%%%%%%%%%%%%%%%%%%%%%%%%%%%%%%%%%%%%%%%%%%%%%%%%

%\vspace{-0.1in}
\section{The Complexity of the Linear Best-Arm Identification Problem}\label{s:oracle}
\vspace{-0.1in}

%\begin{figure}[!ht]
\begin{wrapfigure}[15]{rh}{0.40\textwidth}
\vspace{-0.15in}
\centering
\includegraphics[scale = 0.36]{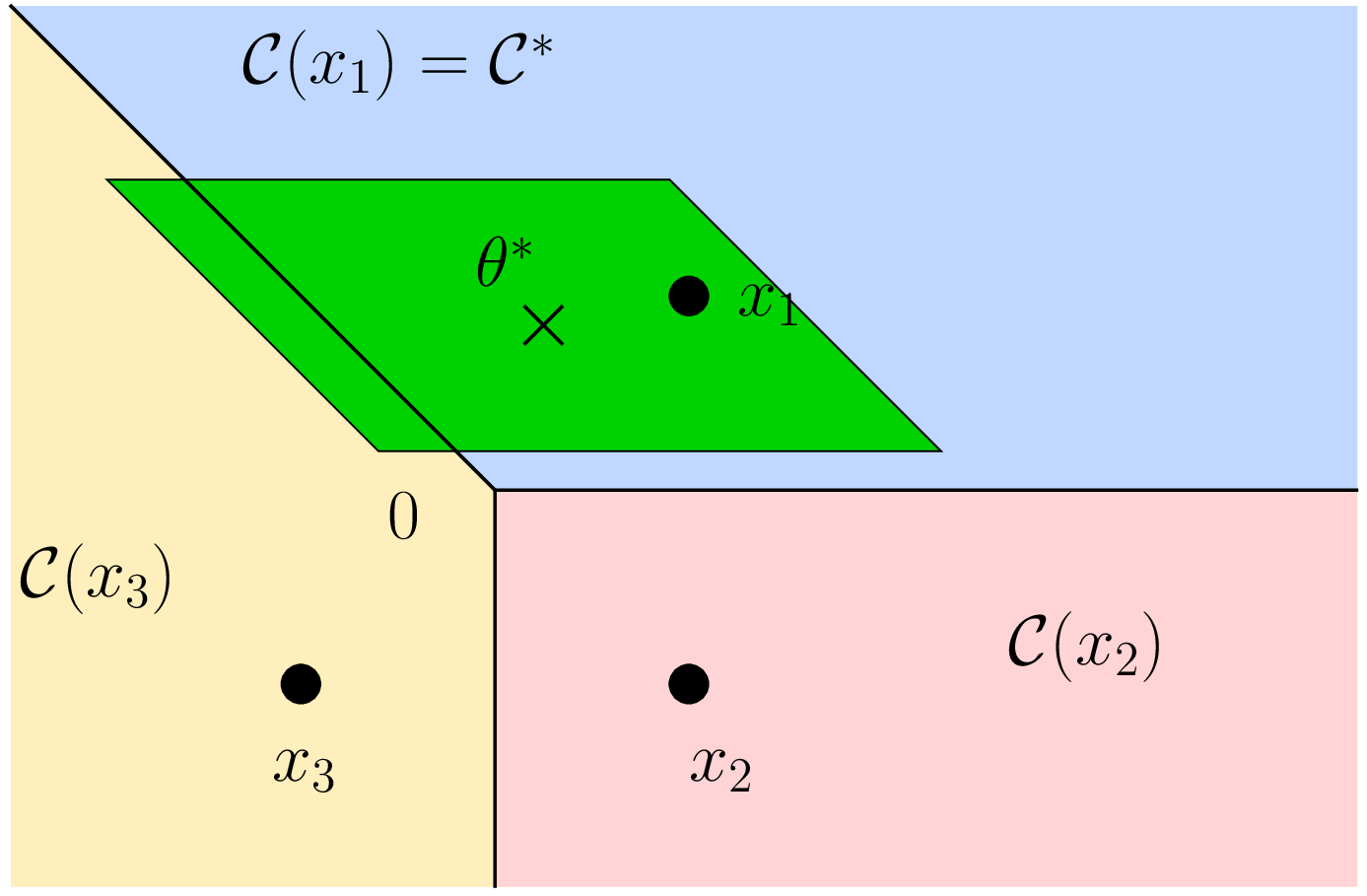}%{pics/conf_set3.eps}
\vspace{-0.05in}
\caption{{\small The cones corresponding to three arms (dots) in $\mathbb R^2$. Since $\theta^* \in \mathcal{C}(x_1)$, then $x^* =x_1$. The confidence set $\S^*(\bx_n)$ (in green) is aligned with directions $x_1-x_2$ and $x_1-x_3$. Given the uncertainty in $\S^*(\bx_n)$, both $x_1$ and $x_3$ may be optimal.}}
\label{f:cones}
\end{wrapfigure}
%\end{figure}

%\begin{figure}[!ht]
%\begin{minipage}[c]{0.475\textwidth}
%\centering
%\includegraphics[scale = 0.4]{pics/cones2.eps}
%\caption{\small{Partition of the space in the cones corresponding to three arms: $x_1, x_2, x_3 \in \mathbb{R}^2$. Here $\theta^* \in \mathcal{C}(x_1)$, thus $x^* =x_1$.}}
%\label{f:cones}
%%\vspace{-0.5cm}
%\end{minipage}
%\hfill
%\begin{minipage}[c]{0.475\textwidth}
%\centering
%\includegraphics[scale = 0.4]{pics/conf_set3.eps}
%\caption{\small{Illustration of the confidence set $\S^*$ (green area) constructed by the oracle algorithm, which knowns $\mathcal{C}(x^*)=\mathcal{C}(x_1)$ and only samples directions $x_1-x_2$ and $x_1-x_3$.}}
% \label{fig:conf-set1}
%%\vspace{-0.5cm}
%\end{minipage}
%\end{figure}

As reviewed in Sect.~\ref{s:prelim}, in the MAB case the complexity of the best-arm identification task is characterized by the reward gaps between the optimal and suboptimal arms. In this section, we propose an extension of the notion of complexity to the case of linear best-arm identification. In particular, we characterize the complexity by the performance of an \textit{oracle} with access to the parameter $\theta^*$.

\textbf{Stopping condition.} Let $\C(x) \!=\! \{\theta\in\Re^d, x\in \Pi(\theta)\}$ be the set of parameters $\theta$ which admit $x$ as an optimal arm. As illustrated in Fig.~\ref{f:cones}, $\C(x)$ is the cone defined by the intersection of half-spaces such that $\C(x) = \cap_{x'\in\X} \{\theta\in\Re^d, (x-x')^\top \theta \geq 0 \}$ and all the cones together form a partition of the Euclidean space $\Re^d$. We assume that the oracle knows the cone $\C(x^*)$ containing all the parameters for which $x^*$ is optimal. Furthermore, we assume that for any allocation $\bx_n$, it is possible to construct a confidence set $\S^*(\bx_n)\subseteq \Re^d$ such that $\theta^*\in\S^*(\bx_n)$ and the (random) OLS estimate $\htheta_n$ belongs to $\S^*(\bx_n)$ with high probability, i.e., $\Prob\big( \htheta_n \in \S^*(\bx_n)\big) \geq 1-\delta$.
%
%\begin{align}\label{eq:oracle.confidence.set}
%\Prob\big( \htheta_n \in \S^*(\bx_n)\big) \geq 1-\delta.
%\end{align}
%
As a result, the oracle stopping criterion simply checks whether the confidence set $\S^*(\bx_n)$ is contained in $\C(x^*)$ or not. In fact, whenever for an allocation $\bx_n$ the set $\S^*(\bx_n)$ overlaps the cones of different arms $x\in\X$, there is ambiguity in the identity of the arm $\Pi(\htheta_n)$. On the other hand when all possible values of $\htheta_n$ are included with high probability in the ``right'' cone $C(x^*)$, then the optimal arm is returned.

\begin{lemma}\label{lem:oracle.stop}
%Let $\bx_n\!=\!(x_1, \dots, x_n)$ be an allocation such that $\S^*(\bx_n) \subseteq \C(x^*)$. Then $\Prob\big( \Pi(\htheta_n) \neq x^* \big) \leq \delta$.
Let $\bx_n$ be an allocation such that $\S^*(\bx_n) \subseteq \C(x^*)$. Then $\Prob\big( \Pi(\htheta_n) \neq x^* \big) \leq \delta$.
%%
%\begin{align*}
%\Prob\big( \Pi(\htheta_n) \neq x^* \big) \leq \delta.
%\end{align*}
%%
\end{lemma}

\textbf{Arm selection strategy.} From the previous lemma\footnote{For all the proofs in this paper, we refer the reader to the long version of the paper~\citep{tr_linear_best}.} 
it follows that the objective of an arm selection strategy is to define an allocation $\bx_n$ which leads to $\S^*(\bx_n) \subseteq \C(x^*)$ as quickly as possible.\footnote{Notice that by definition of the confidence set and since $\theta_n \rightarrow \theta^*$ as $n\rightarrow\infty$, any strategy repeatedly pulling all the arms would eventually meet the stopping condition.} Since this condition only depends on deterministic objects ($\S^*(\bx_n)$ and $\C(x^*)$), it can be computed independently from the actual reward realizations. From a geometrical point of view, this corresponds to choosing arms so that the confidence set $\S^*(\bx_n)$ shrinks into the optimal cone $\C(x^*)$ within the smallest number of pulls. To characterize this strategy we need to make explicit the form of $\S^*(\bx_n)$. %In Fig.~\ref{f:allocation} we provide a qualitative illustration of different definitions of $\S^*(\bx_n)$. 
Intuitively speaking, the more $\S^*(\bx_n)$ is ``aligned'' with the boundaries of the cone, the easier it is to shrink it into the cone. More formally, the condition $\S^*(\bx_n) \subseteq \C(x^*)$ is equivalent to
\begin{align*}%\label{eq:oracle.stop.condition}
\forall x\in\X, \forall \theta\in\S^*(\bx_n), (x^*-x)^\top \theta \geq 0 \enspace\Leftrightarrow\enspace \forall y\in\Y^*, \forall \theta\in\S^*(\bx_n), y^\top(\theta^*-\theta) \leq \Delta(y).
\end{align*}
Then we can simply use Prop.~\ref{p:bound} to directly control the term $y^\top(\theta^*-\theta)$ and define
\begin{align}\label{eq:oracle.conf.set}
\S^*(\bx_n) = \left \{\theta\in\Re^d, \forall y\in\Y^*, y^\top(\theta^* -\theta) \leq c ||y||_{A_{\bx_n}^{-1}} \sqrt{\log_n(K^2/\delta)} \right\}.
\end{align}
Thus the stopping condition $\S^*(\bx_n) \subseteq \C(x^*)$ is equivalent to the condition that, for any $y\in\Y^*$,
\begin{align}\label{eq:oracle.stop.condition.2}
%\forall \theta\in\S^*(\bx_n), y^\top(\theta^* -\htheta_n) \leq c ||y||_{A_{\bx_n}^{-1}} \sqrt{\log_n(K^2/\delta)} \leq \Delta(y).
c ||y||_{A_{\bx_n}^{-1}} \sqrt{\log_n(K^2/\delta)} \leq \Delta(y).
\end{align}
From this condition, the oracle allocation strategy simply follows as
\begin{align}\label{eq:oracle.strategy}
\bx_n^* = \arg\min_{\bx_n} \max_{y\in\Y^*} \frac{c ||y||_{A_{\bx_n}^{-1}} \sqrt{\log_n(K^2/\delta)}}{\Delta(y)} = \arg\min_{\bx_n} \max_{y\in\Y^*}\frac{||y||_{A_{\bx_n}^{-1}}}{\Delta(y)}. %= \arg\min_{\bx_n} \rho^*(\bx_n),
\end{align}
%
%where $\rho^*(\bx_n) = \max_{y\in\Y^*} ||y||_{A_{\bx_n}^{-1}} / \Delta(y)$.
Notice that this strategy does not return an uniformly accurate estimate of $\theta^*$ but it rather pulls arms that allow to reduce the uncertainty of the estimation of $\theta^*$ over the directions of interest (i.e., $\Y^*$) below their corresponding gaps. This implies that the objective of Eq.~\ref{eq:oracle.strategy} is to exploit the global linear assumption by pulling any arm in $\X$ that could give information about $\theta^*$ over the directions in $\Y^*$, so that directions with small gaps are better estimated than those with bigger gaps. 

\textbf{Sample complexity.} We are now ready to define the sample complexity of the oracle, which %, which will work as a reference for all the algorithms developed in the sequel. 
corresponds to the minimum number of steps needed by the allocation in Eq.~\ref{eq:oracle.strategy} to achieve the stopping condition in Eq.~\ref{eq:oracle.stop.condition.2}. 
%Formally, this corresponds to
%%
%\begin{align}\label{eq:oracle.complexity.old}
%\tilde N^*= \min\bigg\{n\in\mathbb{N}, \min_{\bx_n}\max_{y\in\Y^*} \frac{c ||y||_{A_{\bx_n}^{-1}} \sqrt{\log_n(K^2/\delta)}}{\Delta(y)} \leq 1\bigg\}.
%\end{align}
%
%%
%\begin{align}\label{eq:oracle.complexity}
%N^*= \min\bigg\{n\in\mathbb{N}, \min_{\bx_n}\max_{y\in\Y^*} \frac{c ||y||_{A_{\bx_n}^{-1}} \sqrt{\log_n(K^2/\delta)}}{\Delta(y)} \leq 1\bigg\} = \min\Big\{n\in\mathbb{N}, c\rho_n^*\sqrt{\log_n(K^2/\delta)} \leq 1\Big\},
%\end{align}
%%
%where $\rho^*_n = \rho^*(\bx_n^*)$. 
From a technical point of view, it is more convenient to express the complexity of the problem in terms of the optimal design (soft allocation) instead of the discrete allocation $\bx_n$. Let $\rho^*(\lambda) =  \max_{y\in\Y^*} ||y||^2_{\Lambda_{\lambda}^{-1}} / \Delta^2(y)$ be the square of the objective function in Eq.~\ref{eq:oracle.strategy} for any design $\lambda\in\D^k$. We define the complexity of a linear best-arm identification problem as the performance achieved by the optimal design $\lambda^* = \arg\min_\lambda \rho^*(\lambda)$, i.e.
\begin{equation}\label{eq:problem.complexity}
H_{\text{LB}} = \min_{\lambda\in\D^k} \max_{y\in\Y^*} \frac{||y||^2_{\Lambda_{\lambda}^{-1}}}{\Delta^2(y)} = \rho^*(\lambda^*).
\end{equation}
This definition of complexity is less explicit than in the case of $H_{\text{MAB}}$ but it contains similar elements, notably the inverse of the gaps squared. Nonetheless, instead of summing the inverses over all the arms, $H_{\text{LB}}$ implicitly takes into consideration the correlation between the arms in the term $||y||^2_{\Lambda_{\lambda}^{-1}}$, which represents the uncertainty in the estimation of the gap between $x^*$ and $x$ (when $y=x^*-x$). As a result, from Eq.~\ref{eq:oracle.stop.condition.2} the sample complexity becomes
\begin{align}\label{eq:oracle.complexity}
N^* = c^2 H_{\text{LB}} \log_n(K^2/\delta),
\end{align}
where we use the fact that, if implemented over $n$ steps, $\lambda^*$ induces a design matrix $A_{\lambda^*} = n\Lambda_{\lambda^*}$ and $\max_{y}||y||^2_{A_{\lambda^*}^{-1}}/\Delta^2(y) = \rho^*(\lambda^*)/n$. %Notice that $N^*$ is always smaller than the number of samples defined in Eq.~\ref{eq:oracle.complexity.old}. 
Finally, we bound the range of the complexity.% $H_{\text{MAB}}$.

\begin{lemma}\label{lem:complexity.range}
Given an arm set $\X \subseteq \Re^d$ and a parameter $\theta^*$, the complexity $H_{\text{LB}}$ (Eq.~\ref{eq:problem.complexity}) is such that%\todoaout{Fix the lower bound.}
\begin{align}\label{eq:complexity.range}
%\frac{C}{\Delta_{\min}^2} \leq H_{\text{LB}} \leq \frac{2d}{\Delta_{\min}^2}.
\max_{y\in\Y^*}||y||^2 / (L\Delta_{\min}^2)  \leq H_{\text{LB}} \leq 4d/\Delta_{\min}^2.
\end{align}
Furthermore, if $\X$ is the canonical basis, the problem reduces to a MAB and $H_{\text{MAB}} \!\leq\! H_{\text{LB}} \!\leq\! 2H_{\text{MAB}}$.
\end{lemma}

The previous bounds show that $\Delta_{\min}$ plays a significant role in defining the complexity of the problem, while the specific shape of $\X$ impacts the numerator in different ways. In the worst case the full dimensionality $d$ appears (upper-bound), and more arm-set specific quantities, such as the norm of the arms $L$ and of the directions $\Y^*$, appear in the lower-bound.

%Notice that $N^*$ is always smaller than the number of samples defined in Eq.~\ref{eq:oracle.complexity.old}.
%%
%\begin{align}\label{eq:oracle.complexity}
%N^*= \min\bigg\{n\in\mathbb{N}, \frac{c H_{\text{LB}} \sqrt{\log_n(K^2/\delta)}}{n} \leq 1\bigg\},
%\end{align}
%%
%The sample complexity can also be characterized in terms of the optimal sequence of pulls to the arms that leads to meet the stopping condition as fast as possible. Let $\lambda^* \in \mathbb R^k$ be the vector of optimal allocation proportions to the arms in $\X$, where $\sum_{i=1}^k \lambda(i) = 1$ and $ \lambda(i) \geq 0 $ is the proportion of pulls to arm $x_i$. Then, we define the problem-dependent complexity of the linear best-arm identification problem by the quantity
%\begin{equation}
%H_{\text{LB}}= \lambda^* = \min_\lambda \rho_n^*.
%\end{equation}
%

%\todoa{Define the complexity $H$ (without $\log_n(K^2/\delta)$).}

%%%%%%%%%%%%%%%%%%%%%%%%%%%%%%%%%%%%%%%%%%%%%%%%%%%%%%%%%%%%%%%%%%%%%%%%%%%%%%%
%% STATIC ALLOCATION STRATEGIES
%%%%%%%%%%%%%%%%%%%%%%%%%%%%%%%%%%%%%%%%%%%%%%%%%%%%%%%%%%%%%%%%%%%%%%%%%%%%%%%

\vspace{-0.1in}
\section{Static Allocation Strategies}\label{s:static}
\vspace{-0.1in}

%\begin{algorithm}[H]
\begin{wrapfigure}[13]{rh}{0.47\textwidth}
\begin{minipage}[t]{1.0\linewidth}
\vspace{-0.3in}
\begin{small}
\bookboxx{
\begin{algorithmic}
	\State \textbf{Input:} decision space $\X \in \Re^d$, confidence $\delta > 0$
	\State Set: $t = 0;~ Y = \{y =(x - x');  x\neq x' \in \X \};$ 
	\While {Eq.~\ref{eq:emp.stop.condition.2} is not true}
		\If {$G$-allocation}
			%\State Select arm $x_t$ (Eq.~\ref{eq:g.optimal.design.incremental} or relaxation in Eq.~\ref{eq:g.optimal.design.relax}) 
			\\ \hspace{0.4cm} $x_t = \underset{x \in X}{\argmin}~\underset{x' \in X}{\max}~ x'^\top (A + xx^\top)^{-1} x' $
		\ElsIf{$\X\Y$-allocation}
			%\State Select arm $x_t$ (Eq.~\ref{eq:xy.optimal.design.incremental} or relaxation in Eq.~\ref{eq:xy.optimal.design.relax})
			\\ \hspace{0.4cm} $x_t = \underset{x \in X}{\argmin}~\underset{y \in Y}{\max}~ y^\top (A + xx^\top)^{-1} y $
		\EndIf
%		\STATE Update $A_t = \sum\limits_{s=1}^{t} x_s x_s^\top;~ b_t=\sum\limits_{s=1}^{t} x_s r_s;~ \htheta_t=A_t^{-1}b_t$
		\State Update $\htheta_t=A_t^{-1}b_t$, $t=t+1$
%		\STATE $t=t+1$
	\EndWhile
	\State Return arm $\Pi(\htheta_t)$
\end{algorithmic}
}
\end{small}
\end{minipage}
\vspace{-0.18in}
\caption{{\small Static allocation algorithms}}
\label{alg:fixed}
\end{wrapfigure}

%\end{algorithm}

The oracle stopping condition (Eq.~\ref{eq:oracle.stop.condition.2}) and allocation strategy (Eq.~\ref{eq:oracle.strategy}) cannot be implemented in practice since $\theta^*$, the gaps $\Delta(y)$, and the directions $\Y^*$ are unknown. In this section we investigate how to define algorithms that only rely on the information available from $\X$ and the samples collected over time. We introduce an empirical stopping criterion and two static allocations.

\textbf{Empirical stopping criterion.} The stopping condition $\S^*(\bx_n) \subseteq \C(x^*)$ cannot be tested since $\S^*(\bx_n)$ is centered in the unknown parameter $\theta^*$ and $\C(x^*)$ depends on the unknown optimal arm $x^*$. Nonetheless, we notice that given $\X$, for each $x\in\X$ the cones $\C(x)$ can be constructed beforehand. Let $\hS(\bx_n)$ be a high-probability confidence set such that for any $\bx_n$, $\htheta_n \in \hS(\bx_n)$ and $\Prob(\theta^*\in\hS(\bx_n)) \geq 1-\delta$. Unlike $\S^*$, $\hS$ can be directly computed from samples and we can stop whenever there exists an $x$ such that $\hS(\bx_n) \subseteq \C(x)$.

\begin{lemma}\label{lem:emp.stop}
Let $\bx_n=(x_1,\ldots,x_n)$ be an arbitrary allocation sequence. If after $n$ steps there exists an arm $x\in\X$ such that $\hS(\bx_n) \subseteq \C(x)$ then $\Prob\big( \Pi(\htheta_n) \neq x^* \big) \leq \delta$. 
%%
%\begin{align*}
%\Prob\big( \Pi(\htheta_n) \neq x^* \big) \leq \delta.
%\end{align*}
%%
\end{lemma}

\textbf{Arm selection strategy.} Similarly to the oracle algorithm, we should design an allocation strategy that guarantees that the (random) confidence set $\hS(\bx_n)$ shrinks in one of the cones $\C(x)$ within the fewest number of steps. Let $\hDelta_n(x,x') = (x-x')^\top \htheta_n$ be the empirical gap between arms $x, x'$. Then the stopping condition $\hS(\bx_n) \subseteq \C(x)$ can be written as
\begin{align}\label{eq:emp.stop.condition}
\exists x\in\X, \forall x'\in\X, &\forall \theta\in\hS(\bx_n), (x-x')^\top \theta \geq 0 \nonumber\\
&\Leftrightarrow \enspace\exists x\in\X, \forall x'\in\X, \forall \theta\in\hS(\bx_n), (x-x')^\top(\htheta_n-\theta) \leq \hDelta_n(x,x').
\end{align}
This suggests that the empirical confidence set can be defined as
\begin{align}\label{eq:emp.conf.set}
\hS(\bx_n) = \left \{\theta\in\Re^d, \forall y\in\Y, y^\top(\htheta_n -\theta) \leq c ||y||_{A_{\bx_n}^{-1}} \sqrt{\log_n(K^2/\delta)}\right\}.
\end{align}
Unlike $\S^*(\bx_n)$, $\hS(\bx_n)$ is centered in $\htheta_n$ and it considers all directions $y\in\Y$. As a result, the stopping condition in Eq.~\ref{eq:emp.stop.condition} could be reformulated as
\begin{align}\label{eq:emp.stop.condition.2}
\exists x\in\X, \forall x'\in\X,  c ||x-x'||_{A_{\bx_n}^{-1}} \sqrt{\log_n(K^2/\delta)} \leq \hDelta_n(x,x').
\end{align}
Although similar to Eq.~\ref{eq:oracle.stop.condition.2}, unfortunately this condition cannot be directly used to derive an allocation strategy. In fact, it is considerably more difficult to define a suitable allocation strategy to fit a random confidence set $\hS$ into a cone $\C(x)$ for an $x$ which is not known in advance. 
%As a result, we do not know on which pairs $(x,x')$ the allocation should focus on.
% and if we try to make the accuracy $c ||x-x'||_{A_{\bx_n}^{-1}} \sqrt{\log_n(K^2/\delta)}$ smaller than the random empirical gaps $\hDelta_n(x,x')$ for any $(x,x')\in\X^2$, we would obtain a very poor allocation, since for some pairs $(x,x')$ with $x,x'\neq x^*$ the gap $\Delta(x,x')$ may be zero, thus forcing the allocation to concentrate resources to increase the accuracy in the attempt to discriminate between arms that are equivalent. 
In the following we propose two allocations that try to achieve the condition in Eq.~\ref{eq:emp.stop.condition.2} as fast as possible by implementing a static arm selection strategy, while we present a more sophisticated adaptive strategy in Sect.~\ref{s:xy.adaptive}. The general structure of the static allocations in summarized in Fig.~\ref{alg:fixed}.

%%%%%%%%%%%%%%%%%%%%%%%%%%%%%%%%%%%%%%%%%%%%%%%%%%%%%%%%%%%%%%%%%%%%%%%%%%%%%%%
%% G-ALLOCATION STRATEGY
%%%%%%%%%%%%%%%%%%%%%%%%%%%%%%%%%%%%%%%%%%%%%%%%%%%%%%%%%%%%%%%%%%%%%%%%%%%%%%%

%\subsection{The $G$-Allocation Strategy}\label{ss:g.design}
\textbf{$\mathbf{G}$-Allocation Strategy.}
The definition of the $G$-allocation strategy directly follows from the observation that for any pair $(x,x')\in\X^2$ we have that $||x-x'||_{A_{\bx_n}^{-1}} \leq 2\max_{x''\in\X} ||x''||_{A_{\bx_n}^{-1}}$.
%%
%\begin{align}\label{eq:triang.inequality}
%||x-x'||_{A_{\bx_n}^{-1}} \leq ||x||_{A_{\bx_n}^{-1}} + ||x'||_{A_{\bx_n}^{-1}} \leq 2\max_{x''\in\X} ||x''||_{A_{\bx_n}^{-1}}.
%\end{align}
%%
This suggests that an allocation minimizing $\max_{x\in\X} ||x||_{A_{\bx_n}^{-1}}$ reduces an upper bound on the quantity tested in the stopping condition in Eq.~\ref{eq:emp.stop.condition.2}. 
%If such upper bound is not too loose, then we expect the stopping condition to be triggered in relatively few steps. 
%Building on this observation, 
Thus, for any fixed $n$, we define the $G$-allocation as
\begin{align}\label{eq:g.optimal.design}
\bx_n^G = \arg\min_{\bx_n} \max_{x\in\X} ||x||_{A_{\bx_n}^{-1}}.
\end{align}

We notice that this formulation coincides with the standard $G$-optimal design (hence the name of the allocation) defined in experimental design theory~\citep[Sect.~9.2]{pukelsheim2006optimal} to minimize the maximal mean-squared prediction error in linear regression. %(see Eq.~\ref{eq:mse.ols}) 
The $G$-allocation can be interpreted as the design that allows to estimate $\theta^*$ \textit{uniformly well} over all the arms in $\X$. %Hence, in the MAB setting where the arms are independent, $G$-allocation smoothly reduces to a uniform allocation over arms. 
Notice that the $G$-allocation in Eq.~\ref{eq:g.optimal.design} is well defined only for a fixed number of steps $n$ and it cannot be directly implemented in our case, since $n$ is unknown in advance. Therefore we have to resort to a more ``incremental'' implementation. %of the $G$-allocation.
In the experimental design literature a wide number of approximate solutions have been proposed to solve the $\mathit{NP}$-hard discrete optimization problem in Eq.~\ref{eq:g.optimal.design} (see 
\citep{bouhtou2010submodularity,sagnol2013approximation} for some recent results and %Appendix B
\citep{tr_linear_best} for a more thorough discussion). 
%------------------------
%Here, we use the rounding procedure defined in~\citep[Chapter~12]{pukelsheim2006optimal}, with a performance guarantee bounded by a factor $\beta = \big(1+ \frac{d+d^2+2}{2t}\big)d, \forall t \geq d$. Then, the sample complexity of the $G$-allocation strategy, denoted by $N^G$, is bounded as follows.
%-----------------------
For any approximate $G$-allocation strategy with performance no worse than a factor $(1+ \beta)$ of the optimal strategy $\bx_n^G$, the sample complexity $N^G$ is bounded as follows.

%Given the performance guarantees in , we can easily derive the following sample complexity bound for the $G$-allocation.

%
%See also Fig.~\ref{f:conf_g} for a geometrical illustration of the confidence set which is constructed by using this allocation. 
%
%
%%
%\begin{figure}[!ht]
%\begin{minipage}[c]{0.4\textwidth}
%%\vspace{-0.7cm}
%\centering
%\includegraphics[scale = 0.4]{pics/ev-g.eps}
%%\vspace{-0.5cm}
%\caption{\small{Illustration of the change in the shape of the confidence set $\S$ obtained by $G$, after 100, 500, and 1000 samples. Notice that all direction remain equally important during the sampling allocation.}}
%\label{f:conf_g}
%\end{minipage}
%\hfill
%\begin{minipage}[c]{0.57\textwidth}
%%\vspace{-0.7cm}
%\centering
%\includegraphics[scale = 0.4]{pics/ev-xy-ad.eps}
%%\vspace{-0.5cm}
%\caption{\small{Illustration of the change in the shape of the confidence set $\S$ obtained by $XY$-Adaptive, after 100, 500, and 1000 samples. As opposed to $G$, $XY$-Adaptive samples directions $x_1 - x_3$ and $x_1-x_2$ most of the time, which leads to progressively approach the shape of the confidence $\mathcal S^*$ (see Fig.~\ref{fig:conf-set1})}.}
% \label{fig:conf_xy_ad}
%\end{minipage}
%\end{figure}
%%
%
%\textbf{Sample complexity.} 

\begin{theorem}\label{thm:g.allocation}
If the $G$-allocation strategy is implemented with a $\beta$-approximate method %(see lemmas~\ref{lem:g.optimal.continuous.relax} and~\ref{lem:g.optimal.increment} in the supplement) 
and the stopping condition in Eq.~\ref{eq:emp.stop.condition.2} is used, then
\begin{align}\label{eq:bound.g.allocation}
\Prob\bigg[ N^G \leq \frac{16c^2 d (1+\beta)\log_n(K^2/\delta)}{\Delta_{\min}^2} \wedge \Pi(\htheta_{N^G}) = x^* \bigg] \geq 1-\delta.
\end{align}
\end{theorem}

Notice that this result matches (up to constants) the worst-case value of $N^*$ given the upper bound on $H_{\text{LB}}$. This means that, although completely static, the $G$-allocation is already worst-case optimal.

%%%%%%%%%%%%%%%%%%%%%%%%%%%%%%%%%%%%%%%%%%%%%%%%%%%%%%%%%%%%%%%%%%%%%%%%%%%%%%%
%% XY-ALLOCATION STRATEGY
%%%%%%%%%%%%%%%%%%%%%%%%%%%%%%%%%%%%%%%%%%%%%%%%%%%%%%%%%%%%%%%%%%%%%%%%%%%%%%%

%\subsection{$\X\Y$-Allocation Strategy}\label{ss:xy.design}
\textbf{\boldmath{$\X\Y$}-Allocation Strategy.}\label{ss:xy.design.}
Despite being worst-case optimal, $G$-allocation is minimizing a rather loose upper bound on the quantity used to test the stopping criterion. Thus, we define an alternative static allocation that targets the stopping condition in Eq.~\ref{eq:emp.stop.condition.2} more directly by reducing its left-hand-side for any possible direction in $\Y$. For any fixed $n$, we define the $\X\Y$-allocation as
\begin{align}\label{eq:xy.optimal.design}
\bx_n^{\X\Y} = \arg\min_{\bx_n} \max_{y\in\Y} ||y||_{A_{\bx_n}^{-1}}.
\end{align}
%
%In the following we denote by $\rho^{\X\Y}(\bx_n) = \max_{y\in\X} \sqrt{y^\top A_{\bx_n}^{-1}y}$ the performance obtained by applying the $\X\Y$-allocation over $n$ steps. 
%Unlike the $G$-allocation, here the objective is not to estimate $\theta^*$ equally well for the prediction of the value of each arm $x\in\X$. On the other hand, it 
$\X\Y$-allocation is based on the observation that the stopping condition in Eq.~\ref{eq:emp.stop.condition.2} requires only the empirical gaps $\widehat \Delta (x,x')$ to be well estimated, %which corresponds to increasing the accuracy of the estimation of the empirical gaps computed as $y^\top \htheta_n$. An interesting feature of Eq.~\ref{eq:xy.optimal.design} is that 
hence arms are pulled with the objective of increasing the accuracy of directions in $\Y$ instead of arms $\X$. This problem can be seen as a transductive variant of the $G$-optimal design~\citep{transductiveED}, where the target vectors $\Y$ are different from the vectors $\X$ used in the design. The sample complexity of the $\X\Y$-allocation is as follows.

\begin{theorem}\label{thm:xy.allocation}
If the $\X\Y$-allocation strategy is implemented with a $\beta$-approximate method %(see lemmas~\ref{lem:g.optimal.continuous.relax} and~\ref{lem:g.optimal.increment} in the supplement) 
and the stopping condition in Eq.~\ref{eq:emp.stop.condition.2} is used, then
\begin{align}\label{eq:bound.xy.allocation}
\Prob\bigg[ N^{\X\Y} \leq \frac{32c^2 d (1+\beta)\log_n(K^2/\delta)}{\Delta_{\min}^2} \wedge \Pi(\htheta_{N^{\X\Y}}) = x^* \bigg] \geq 1-\delta.
\end{align}
\end{theorem}

Although the previous bound suggests that $\X\Y$ achieves a performance comparable to the $G$-allocation, %in Sect.~\ref{s:g.vs.xy.allocation} of the supplement we show a simple example when $\X\Y$ may be arbitrarily better than $G$-allocation.
in fact $\X\Y$ may be arbitrarily better than $G$-allocation (for an example, see 
\citep{tr_linear_best}).

%\todoa{I tried to read the paper~\citep{transductiveED} but I haven't understood whether it could be actually useful for us (at least for the implementation part).}
% I also think that we can't use that paper too much for the implementation part. However, the explanations they gave about the problem and expected results might be useful. 

%%%%%%%%%%%%%%%%%%%%%%%%%%%%%%%%%%%%%%%%%%%%%%%%%%%%%%%%%%%%%%%%%%%%%%%%%%%%%%%
%% XY-ADAPTIVE DESIGN
%%%%%%%%%%%%%%%%%%%%%%%%%%%%%%%%%%%%%%%%%%%%%%%%%%%%%%%%%%%%%%%%%%%%%%%%%%%%%%%

\vspace{-0.1in}
\section{\boldmath{$\X\Y$}-Adaptive Allocation Strategy}\label{s:xy.adaptive}
\vspace{-0.1in}

\begin{wrapfigure}[23]{rh}{0.56\textwidth}
\vspace{-0.25in}
\begin{minipage}[t]{1.0\linewidth}
\begin{small}
\bookboxx{
\begin{algorithmic}
	\State \textbf{Input:} decision space $\X\! \in\! \Re^d$; parameter $\alpha$; confidence $\delta$
	\State Set $j \!=\! 1;~ \hX_j\!=\! \X;~ \hY_1 \!=\! \Y;~ \rho_0 \!=\! 1;~ n_0 \!=\! d(d+1)+1$ 
	\While {$|\hX_j| > 1 $}
		\State $\rho^j = \rho^{j-1}$
		\State $t = 1; A_0=I$
		\While{$\rho^j/t \geq \alpha\rho^{j-1}(\bx_{n_{j-1}}^{j-1})/n_{j-1}$}
			\State Select arm % $x_t$ (Eq.~\ref{eq:xy.optimal.design.incremental} or relaxation in Eq.~\ref{eq:xy.optimal.design.relax})% = \argmin_{x \in X} \max_{y \in Y_j} y^\top (A + xx^\top)^{-1} y $\\
 $x_t = \underset{x \in X}{\argmin}~\underset{y \in Y}{\max}~ y^\top (A + xx^\top)^{-1} y $
			\State Update $A_t = A_{t-1}+ x_t x_t^\top$, $t=t+1$ 
			\State $\rho^j = \max_{y \in \hY_j} y^\top A_t^{-1} y$
%			\STATE $t=t+1$
		\EndWhile
		\State Compute $b=\sum_{s=1}^{t} x_s r_s;~ \htheta_j=A_t^{-1}b$ 
		\State $\hX_{j+1} = \X$
		\For{$x\in\X$}
			\If{$\exists x'\!:\! ||x-x'||_{A_t^{-1}} \sqrt{\log_n(K^2/\delta)} \leq \hDelta_j(x',x)$}
				\State $\hX_{j+1} = \hX_{j+1}-\{x\}$
			\EndIf
		\EndFor
		\State $\hY_{j+1} = \{y =(x - x');  x, x' \in \hX_{j+1}\}$
	\EndWhile
	\State Return $\Pi(\htheta_j)$
\end{algorithmic}
}
\vspace{-0.15in}
\caption{$\X\Y$-Adaptive allocation algorithm}
\label{a:xy.adaptive}
\end{small}
\end{minipage}
\end{wrapfigure}

\textbf{Fully adaptive allocation strategies.} %So far we have only discussed fully static strategies, which only rely on the knowledge of the arm set $\X$. 
Although both $G$- and $\X\Y$-allocation are sound since they minimize upper-bounds on the quantities used by the stopping condition (Eq.~\ref{eq:emp.stop.condition.2}), they may be very suboptimal w.r.t.\ the ideal performance of the oracle introduced in Sec.~\ref{s:oracle}. Typically, an improvement can be obtained by moving to strategies adapting on the rewards observed over time. Nonetheless, as reported in Prop.~\ref{p:adaptive.bound}, whenever $\bx_n$ is not a fixed sequence, %the OLS estimate should be replaced by a regularized estimate and, more important, 
the bound in Eq.~\ref{eq:err.adaptive.bound} should be used. As a result, a factor $\sqrt{d}$  would appear in the definition of the confidence sets and in the stopping condition. This directly implies that the sample complexity of a fully adaptive strategy would scale linearly with the dimensionality $d$ of the problem, thus removing any advantage w.r.t. static allocations. In fact, the sample complexity of $G$- and $\X\Y$-allocation already scales linearly with $d$ and from Lem.~\ref{lem:complexity.range} we cannot expect to improve the dependency on $\Delta_{\min}$. Thus, on the one hand, we need to use the tighter bounds in Eq.~\ref{eq:err.bound} and, on the other hand, we require to be adaptive w.r.t. samples. In the sequel we propose a phased algorithm which successfully meets both requirements using a static allocation within each phase but choosing the type of allocation depending on the samples observed in previous phases.

\textbf{Algorithm.} The ideal case would be to define an empirical version of the oracle allocation in Eq.~\ref{eq:oracle.strategy} so as to adjust the accuracy of the prediction only on the directions of interest $\Y^*$ and according to their gaps $\Delta(y)$. As discussed in Sect.~\ref{s:static} this cannot be obtained by a direct adaptation of Eq.~\ref{eq:emp.stop.condition.2}. 
In the following, we describe a safe alternative to adjust the allocation strategy to the gaps.
\begin{lemma}\label{lem:dominated.arm}
Let $\bx_n$ be a fixed allocation sequence and $\htheta_n$ its corresponding estimate for $\theta^*$.  If  an arm $x \in \X$ is such that 
\begin{equation}\label{eq:discard.arm} 
\exists x'\in\X~\text{s.t.}~ c||x'-x||_{A_{\bx_n}^{-1}}\sqrt{\log_n(K^2/\delta)} < \hDelta_n(x',x),
\end{equation}
then arm $x$ is sub-optimal. Moreover, if %the discarding condition in 
Eq.~\ref{eq:discard.arm} is true, we say that $x'$ dominates $x$.
\end{lemma}
Lem.~\ref{lem:dominated.arm} allows to easily construct the set of potentially optimal arms, denoted $\hX(\bx_n)$, by removing from $\X$ all the dominated arms. As a result, we can replace the stopping condition in Eq.~\ref{eq:emp.stop.condition.2}, by just testing whether the number of non-dominated arms $|\hX(\bx_n)|$ is equal to 1, which corresponds to the case where the confidence set is fully contained into a single cone. Using $\hX(\bx_n)$, we construct $\hY(\bx_n) = \{y=x-x'; x,x'\in\hX(\bx_n)\}$, the set of directions along which the estimation of $\theta^*$ needs to be improved to further shrink $\hS(\bx_n)$ into a single cone and trigger the stopping condition. Note that if $\bx_n$ was an adaptive strategy, then we could not use Lem.~\ref{lem:dominated.arm} to discard arms but we should rely on the bound in Prop.~\ref{p:adaptive.bound}. %In order to prevent this, we need $\bx_n$ to be static w.r.t. the samples used in testing the discarding (and stopping) condition.
To avoid this problem,  an effective solution is to run the algorithm through phases. Let $j\in\mathbb{N}$ be the index of a phase and $n_j$ its corresponding length. We denote by $\hX_j$ the set of non-dominated arms constructed on the basis of the samples collected in the phase $j-1$. This set is used to identify the directions $\hY_j$ and to define a \textit{static} allocation which focuses on reducing the uncertainty of $\theta^*$ along the directions in $\hY_j$. Formally, in phase $j$ we implement the allocation
\begin{align}\label{eq:phase.allocation}
\bx_{n_j}^j = \arg\min_{\bx_{n_j}} \max_{y\in\hY_j} ||y||_{A_{\bx_{n_j}}^{-1}},
\end{align}
which coincides with a $\X\Y$-allocation (see Eq.~\ref{eq:xy.optimal.design}) but restricted on $\hY_j$. Notice that $\bx_{n_j}^j$ may still use any arm in $\X$ which could be useful in reducing the confidence set along any of the directions in $\hY_j$. Once phase $j$ is over, the OLS estimate $\htheta^j$ is computed using the rewards observed within phase $j$ and then is used to test the stopping condition in Eq.~\ref{eq:emp.stop.condition.2}. Whenever the stopping condition does not hold, a new set $\hX_{j+1}$ is constructed using the discarding condition in Lem.~\ref{lem:dominated.arm} and a new phase is started. Notice that through this process, at each phase $j$, the allocation $\bx_{n_j}^j$ is static conditioned on the previous allocations and the use of the bound from Prop.~\ref{p:bound} is still correct. 

A crucial aspect of this algorithm is the length of the phases $n_j$. On the one hand, short phases allow a high rate of adaptivity, since $\hX_j$ is recomputed very often. On the other hand, if a phase is too short, it is very unlikely that the estimate $\htheta^j$ may be accurate enough to actually discard any arm. 
%We notice that ideally a phase should terminate as soon as an arm is discarded from $\hX_j$, i.e., as soon as the uncertainty in estimating the value of an arm falls below its corresponding gaps (see Eq.~\ref{eq:discard.arm}). Nonetheless, this would again make the length of the allocation \textit{adaptive} w.r.t.\ the samples. 
An effective way to define the length of a phase in a deterministic way is to relate it to the actual uncertainty of the allocation in estimating the value of all the active directions in $\hY_j$. In phase $j$, let $\rho^j(\lambda) = \max_{y\in\hY_j} ||y||^2_{\Lambda_{\lambda}^{-1}}$, then given a parameter $\alpha\in(0,1)$, we define
\begin{align}\label{eq:phase.length}
%n_j = \min\bigg\{n\in\mathbb{N}: \frac{\rho^j(\lambda_{\bx^j_n})}{n} \leq \alpha \frac{\rho^{j-1}(\lambda^{j-1})}{n_{j-1}}\bigg\},
n_j = \min\big\{n\in\mathbb{N}: \rho^j(\lambda_{\bx^j_n}) / n \leq \alpha \rho^{j-1}(\lambda^{j-1}) / n_{j-1}\big\},
\end{align}
where $\bx^j_n$ is the allocation defined in Eq.~\ref{eq:phase.allocation} and $\lambda^{j-1}$ is the design corresponding to $\bx_{n_{j-1}}^{j-1}$, the allocation performed at phase $j-1$. In words, $n_j$ is the minimum number of steps needed by the $\X\Y$-adaptive allocation to achieve an uncertainty over all the directions of interest which is a fraction $\alpha$ of the performance obtained in the previous iteration. Notice that given $\hY_j$ and $\rho^{j-1}$ this quantity can be computed before the actual beginning of phase $j$.
The resulting algorithm using the $\X\Y$-Adaptive allocation strategy is summarized in Fig.~\ref{a:xy.adaptive}.

\begin{comment}
\begin{figure}[!ht]
\centering
\includegraphics[scale = 0.4]{pics/conf_set2}
\caption{If at some time $t$, the current confidence set $\S(x_{1:t})$ intercepts $\C(x_2)\cap\C(x_3)$, then any secure algorithm would need to shrink the confidence set along the direction $x_2-x_3$ also, resulting in a final confidence set (green area) smaller than if one would shrink it in directions $x^*-x$ only, as illustrated in Fig.~\ref{fig:conf-set1}.\label{fig:conf-set2}}
\end{figure}
\end{comment}

%****************************
\textbf{Sample complexity.} Although the $\X\Y$-Adaptive allocation strategy is designed to approach the oracle sample complexity $N^*$, in early phases it basically implements a $\X\Y$-allocation and no significant improvement can be expected until some directions are discarded from $\hY$. At that point, $\X\Y$-adaptive starts focusing on directions which only contain near-optimal arms and it starts approaching the behavior of the oracle. As a result, in studying the sample complexity of $\X\Y$-Adaptive we have to take into consideration the unavoidable price % to pay for 
of discarding ``suboptimal'' directions. % which is strictly related to the 
%it resets the estimates at each phase. %This may even result in an overall performance which is worse than the $\X\Y$-allocation itself in early stages. In fact, the $\X\Y$-adaptive allocation deviates from the $\X\Y$-allocation and approaches the behavior of the oracle only as arms are discarded from $\hX$. Unfortunately, this may take some time and it may prevent from achieving a performance close to $N^*$. As a result, we need to introduce a slightly relaxed definition of complexity which will serve as basis for the theoretical analysis of the $\X\Y$-adaptive allocation. 
%Thus, in early phases, $\X\Y$-adaptive might have a performance which is worse then that of $\X\Y$. 
%In fact, 
This cost is directly related to the geometry of the arm space that influences the number of samples needed before arms can be discarded from $\X$. 
%and thus deviates from the $\X\Y$-allocation and approaches the behavior of the oracle. 
To take into account this problem-dependent quantity, we 
introduce a slightly relaxed definition of complexity. % which will serve as basis for the theoretical analysis of the $\X\Y$-adaptive allocation.  
More precisely, we define the number of steps needed to discard all the directions which do not contain $x^*$, i.e. $\Y-\Y^*$. 
From a geometrical point of view, this corresponds to the case when for any pair of suboptimal arms $(x,x')$, the confidence set $\S^*(\bx_n)$ does not intersect the hyperplane separating the cones $\C(x)$ and $\C(x')$. Fig.~\ref{f:cones} offers a simple illustration for such a situation: $\S^*$ no longer intercepts the border line between $\C(x_2)$ and $\C(x_3)$, which implies that direction $x_2 - x_3$ can be discarded.
%In fact, in such a situation $\S^*(\bx_n)$ may still overlap with either $\C(x)$ or $\C(x')$ but we have an accurate enough estimation of $\theta^*$ to rule out which arm between $x$ and $x'$ is better and the corresponding direction $x-x'$ can be then removed from the directions of interest. 
More formally, the hyperplane containing parameters $\theta$ for which $x$ and $x'$ are equivalent is simply $\C(x)\cap \C(x')$ and the quantity
\begin{align}\label{eq:m.value}
M^* = \min\{n\in\mathbb{N}, \forall x\neq x^*, \forall x'\neq x^*, \S^*(\bx_n^{\X\Y}) \cap (\C(x) \cap \C(x')) = \emptyset\}
\end{align}
corresponds to the minimum number of steps needed by the static $\X\Y$-allocation strategy to discard all the \textit{suboptimal} directions. This term together with the oracle complexity $N^*$ characterizes the sample complexity of the phases of the $\X\Y$-adaptive allocation. In fact, the length of the phases is such that either they correspond to the complexity of the oracle or they can never last more than the steps needed to discard all the sub-optimal directions. As a result, the overall sample complexity of the $\X\Y$-adaptive algorithm is bounded as in the following theorem.
%**************************

%Once this condition is achieved, then we can consider the additional number of steps needed to finally shrink the confidence set $\S^*(\bx_n)$ into the cone $\C(x^*)$ and obtain the sample complexity
%%
%\begin{align}\label{eq:relaxed.oracle.complexity}
%\tN^* = M^* + \min\{n\in\mathbb{N}, \S^*(\bx_n^{\X\Y} \cup \bx_n^*) \subseteq C(x^*)\}.
%\end{align}
%%

\begin{theorem}\label{thm:xy.adaptive}
If the $\X\Y$-Adaptive allocation strategy is implemented with a $\beta$-approximate method and the stopping condition in Eq.~\ref{eq:emp.stop.condition.2} is used, then%\todoaout{Fix the statement}
\begin{align}\label{eq:bound.xy.adaptive.allocation}
\Prob\bigg[ N \leq \frac{(1+\beta)\max\{M^*,\frac{16}{\alpha}N^*\}}{\log(1/\alpha)} \log\Big(\frac{c\sqrt{\log_n(K^2/\delta)}}{\Delta_{\min}}\Big)  \wedge \Pi(\htheta_{N}) = x^* \bigg] \geq 1-\delta.
\end{align}
\end{theorem}
%\todoa{From the bound it seems like the a good value for $\alpha$ would be something like $0.3679$ (Matlab optimization).}

We first remark that, unlike $G$ and $\X\Y$, the sample complexity of $\X\Y$-Adaptive does not have any direct dependency on $d$ and $\Delta_{\min}$ (except in the logarithmic term) but it rather scales with the oracle complexity $N^*$ and the cost of discarding suboptimal directions $M^*$. Although this additional cost is probably unavoidable, one may have expected that $\X\Y$-Adaptive may need to discard all the suboptimal directions before performing as well as the oracle, thus having a sample complexity of $O(M^*+N^*)$. Instead, we notice that $N$ scales with the \textit{maximum} of $M^*$ and $N^*$, thus implying that $\X\Y$-Adaptive may actually catch up with the performance of the oracle (with only a multiplicative factor of $16/\alpha$) whenever discarding suboptimal directions is less expensive than actually identifying the best arm. %Finally, we notice that the bound may also provide guidance in setting the parameter $\alpha$.% by minimizing the , which depends on the geometry of the problem, is  

%Relying on the empirical gaps between arms, $\X\Y$-adaptive can obtain an improved sample complexity, 
%if the lengths of the phases allow to exploit the gained knowledge about the gaps. 
%Also, even though $\X\Y$-adaptive is designed to proceed by discarding arms and one would expect its sample complexity to be the sum of the quantities $M^*$ and of $N^*$, the bound suggests that the sample complexity will only be given by the maximum of the two quantities.

%%%%%%%%%%%%%%%%%%%%%%%%%%%%%%%%%%%%%%%%%%%%%%%%%%%%%%%%%%%%%%%%%%%%%%%%%%%%%%%
%% NUMERICAL SIMULATIONS
%%%%%%%%%%%%%%%%%%%%%%%%%%%%%%%%%%%%%%%%%%%%%%%%%%%%%%%%%%%%%%%%%%%%%%%%%%%%%%%

\vspace{-0.1in}
\section{Numerical Simulations}\label{s:experiments}
\vspace{-0.1in}

We illustrate the performance of $\X\Y$-Adaptive and compare it to the $\X\Y$-Oracle strategy (Eq.~\ref{eq:oracle.strategy}), the static allocations $\X\Y$ and $G$, as well as with the fully-adaptive version of $\X\Y$ where $\hX$ is updated at each round and the bound from Prop.\ref{p:adaptive.bound} is used. For a fixed confidence $\delta = 0.05$, we compare the sampling budget needed to identify the best arm with probability at least $1-\delta$.
%In order to stress the different behavior of the algorithms, 
%%%
%We choose a setting where the minimum gap, $\Delta_{\min}$, is much smaller than the other gaps. This allows to illustrate the advantage of the allocation strategies $\X\Y$-Oracle and $\X\Y$-Adaptive that, unlike the static allocations, adapt to the characteristics of the problem (notably the different gaps). %In particular, we illustrate how the difference in the sampling budget changes with the dimensionality of the input space.
%%%
%
%In particular, we want to investigate how the behavior of $\X\Y$-Adaptive moves from a static $\X\Y$-allocation over all the possible direction in the earlier phases towards an allocation more focused on the directions of interest. Finally, we also show how the difference in the sampling budget changes with the dimensionality of the input space.
%
%\textbf{The setting.} 
We consider a set of arms $\X \in \mathbb{R}^d$, with $|\X|=d+1$ including the canonical basis ($e_1,\ldots,e_d$) and an additional arm $x_{\text{d+1}} = [\cos(\omega)~~\sin(\omega)~~0~~\dots~~0]^\top$.
%
%\begin{equation*}
%x_{\text{d+1}} = [\cos(\omega)~~\sin(\omega)~~0~~\dots~~0]^\top.
%\end{equation*}
%
We choose $\theta^* = [2~~0~~0~~\dots~~0]^\top$, and fix $\omega = 0.01$, so that $\Delta_{\min} = (x_1 - x_{\text{d+1}})^\top \theta^*$ is much smaller than the other gaps. In this setting, an efficient sampling strategy should focus on reducing the uncertainty in the direction $\tilde y = (x_1 - x_{\text{d+1}})$ by pulling the arm $x_2=e_2$ which is almost aligned with $\tilde y$. In fact, from the rewards obtained from $x_2$ it is easier to decrease the uncertainty about the second component of $\theta^*$, that is precisely the dimension which allows to discriminate between $x_1$ and $x_{\text{d+1}}$. Also, we fix %the confidence level $\delta = 0.05$,  
$\alpha= 1/10$, and the noise $\eps \sim \mathcal{N}(0,1)$. Each phase begins with an initialization matrix $A_0$, obtained by pulling once each canonical arm. % $d$ times.
In Fig.~\ref{f:budget} we report the sampling budget of the algorithms, averaged over 100 runs, for $d=2\dots10$.

%\begin{figure}[!ht]
\begin{wrapfigure}[15]{rh}{0.45\textwidth}
\vspace{-0.15in}
\centering
\includegraphics[scale = 0.4]{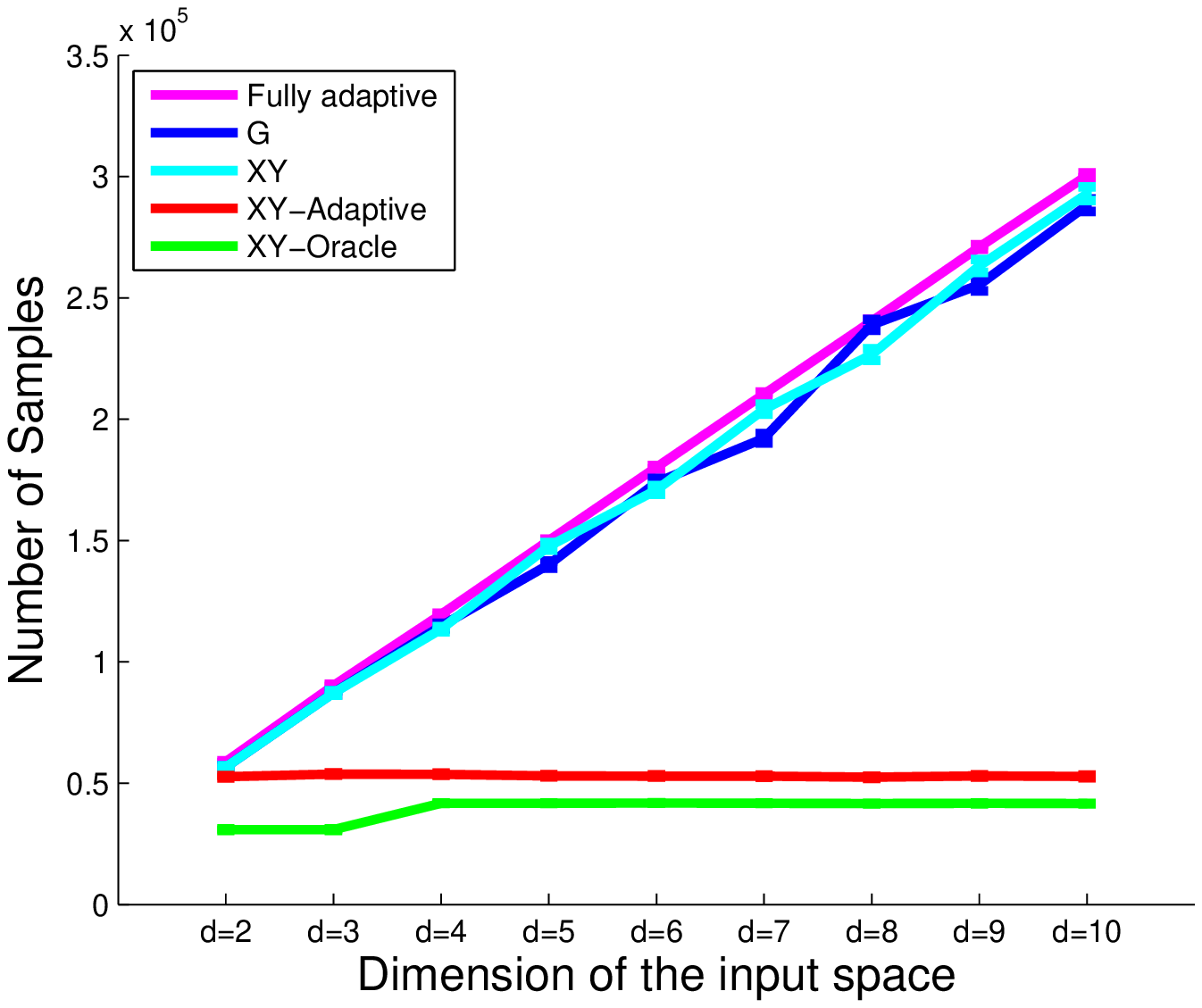}%{pics/all_budget.eps}
\vspace{-0.15in}
\caption{{\small The sampling budget needed to identify the best arm, when the dimension grows from $\mathbb R^2$ to $\mathbb R^{10}.$}}
\label{f:budget}
\end{wrapfigure}

\textbf{The results.} The numerical results show that $\X\Y$-Adaptive is effective in allocating the samples to shrink the uncertainty in the direction $\tilde y$. Indeed, $\X\Y$-adaptive identifies the most important direction after few phases and is able to perform an allocation which mimics that of the oracle. On the contrary, $\X\Y$ and $G$ do not adjust to the empirical gaps and consider all directions as equally important. This behavior forces $\X\Y$ and $G$ to allocate samples until the uncertainty is smaller than $\Delta_\text{min}$ in all directions. Even though the Fully-adaptive algorithm also identifies the most informative direction rapidly, the $\sqrt d$ term in the bound delays the discarding of the arms and prevents the algorithm from gaining any advantage compared to $\X\Y$ and $G$. As shown in Fig.~\ref{f:budget}, the difference between the budget of $\X\Y$-Adaptive and the static strategies increases with the number of dimensions. In fact, while additional dimensions have little to no impact on $\X\Y$-Oracle and $\X\Y$-Adaptive (the only important direction remains $\tilde y$ independently from the number of unknown features of $\theta^*$), for the static allocations more dimensions imply more directions to be considered and 
more features of $\theta^*$ to be estimated uniformly well until the uncertainty falls below $\Delta_{\min}$.

%the more the dimensions, the more the directions considered by the static allocation strategies and the longer it takes until all the dimensions of $\theta^*$ are estimated uniformly well and their uncertainty falls below $\Delta_{\min}$. On the other hand for both $\X\Y$-oracle and $\X\Y$-adaptive the increase in dimension has little to no impact, since the only important direction remains $\tilde y$, independently from the number of unknown features of $\theta^*$ and their sampling budgets are constant over the different dimensions of the problem. 
%

%%%%%%%%%%%%%%%%%%%%%%%%%%%%%%%%%%%%%%%%%%%%%%%%%%%%%%%%%%%%%%%%%%%%%%%%%%%%%%%
%% CONCLUSIONS
%%%%%%%%%%%%%%%%%%%%%%%%%%%%%%%%%%%%%%%%%%%%%%%%%%%%%%%%%%%%%%%%%%%%%%%%%%%%%%%

\vspace{-0.1in}
\section{Conclusions}\label{s:conclusions}
\vspace{-0.1in}

In this paper we studied the problem of best-arm identification  with a fixed confidence, in the linear bandit setting. First we offered a preliminary characterization of the problem-dependent complexity of the best arm identification task and shown its connection with the complexity in the MAB setting. Then, we designed and analyzed efficient sampling strategies for this problem. The $G$-allocation strategy allowed us to point out a close connection with optimal experimental design techniques, and in particular to the G-optimality criterion. Through the second proposed strategy, $\X\Y$-allocation, we  introduced a novel optimal design problem where the testing arms do not coincide with the arms chosen in the design. Lastly, we pointed out the limits that a fully-adaptive allocation strategy might have in the linear bandit setting and proposed a phased-algorithm, $\X\Y$-Adaptive, that learns from previous observations, without suffering from the dimensionality of the problem.
Since this is one of the first works that analyze pure-exploration problems in the linear-bandit setting, it opens the way for an important number of similar problems already studied in the MAB setting. For instance, we can investigate strategies to identify the best-linear arm when having a limited budget or study the best-arm identification when the set of arms is very large (or infinite). Some interesting extensions also emerge from the optimal experimental design literature, such as the study of sampling strategies for meeting the G-optimality criterion when the noise is heterosckedastic, or the design of efficient strategies for satisfying other related optimality criteria, such as V-optimality. 

\vspace{-0.1in}
\paragraph*{Acknowledgments} This work was supported by the French Ministry of Higher Education and Research, Nord-Pas de Calais Regional Council and FEDER through the “Contrat de Projets Etat Region 2007–2013", and European Community’s Seventh Framework Programme under grant agreement no 270327 (project CompLACS).

\bibliographystyle{plain}
%\bibliography{linear-active-learning}	

\newpage

%%%%%%%%%%%%%%%%%%%%%%%%%%%%%%%%%%%%%%%%%%%%%%%%%%%%%%%%%%%%%%%%%%%%%%%%%%%%%%%
%%%%%%%%%%%%%%%%%%%%%%%%%%%%%%%%%%%%%%%%%%%%%%%%%%%%%%%%%%%%%%%%%%%%%%%%%%%%%%%
%%%%%%%%%%%%%%%%%%%%%%%%%%%%%%%%%%%%%%%%%%%%%%%%%%%%%%%%%%%%%%%%%%%%%%%%%%%%%%%
%%%%%%%%%%%%%%%%%%%%%%%%%%%%%%%%%%%%%%%%%%%%%%%%%%%%%%%%%%%%%%%%%%%%%%%%%%%%%%%
%%%%%%%%%%%%%%%%%%%%%%%%%%%%%%%%%%%%%%%%%%%%%%%%%%%%%%%%%%%%%%%%%%%%%%%%%%%%%%%
%%%%%%%%%%%%%%%%%%%%%%%%%%%%%%%%%%%%%%%%%%%%%%%%%%%%%%%%%%%%%%%%%%%%%%%%%%%%%%%
%%%%%%%%%%%%%%%%%%%%%%%%%%%%%%%%%%%%%%%%%%%%%%%%%%%%%%%%%%%%%%%%%%%%%%%%%%%%%%%
%%%%%%%%%%%%%%%%%%%%%%%%%%%%%%%%%%%%%%%%%%%%%%%%%%%%%%%%%%%%%%%%%%%%%%%%%%%%%%%
%%%%%%%%%%%%%%%%%%%%%%%%%%%%%%%%%%%%%%%%%%%%%%%%%%%%%%%%%%%%%%%%%%%%%%%%%%%%%%%

%\begin{comment}

\appendix

%%%%%%%%%%%%%%%%%%%%%%%%%%%%%%%%%%%%%%%%%%%%%%%%%%%%%%%%%%%%%%%%%%%%%%%%%%%%%%%
%% APPENDIX
%%%%%%%%%%%%%%%%%%%%%%%%%%%%%%%%%%%%%%%%%%%%%%%%%%%%%%%%%%%%%%%%%%%%%%%%%%%%%%%

\vspace{-0.1in}
\section{Comparison between $\mathbf{G}$-allocation and \boldmath{$\X\Y$}-allocation}\label{s:g.vs.xy.allocation}
\vspace{-0.1in}

We define two examples illustrating the difference between the $G$ and the $\X\Y$ allocation strategies. Let us consider a problem with $\X\subset \Re^2$ and arms $x_1=[1~~ \epsilon/2]^{\top}$ and $x_2=[1~~ -\epsilon/2]^\top$, where $\epsilon \in (0,1)$. 
%$x_1=[1, \epsilon/2]^{\top}$ and $x_2=[1, -\epsilon/2]^\top$. 
In this case, both static allocations pull the two arms the same number of times, thus inducing an optimal design $\lambda(x_1)=\lambda(x_2)=1/2$. We want to study the (asymptotic) performance of the allocation according to the different definition of error $\max_{x\in\X} x^{\top} \Lambda_{\lambda}^{-1} x$ and $\max_{y\in\Y} y^{\top} \Lambda_{\lambda}^{-1} y$ used by $G$ and $\X\Y$-allocation respectively. We first notice that
\begin{align*}
\Lambda_{\lambda} = \frac{1}{2}
 \left[ \begin{array}{cc}
1 & \epsilon/2 \\
\epsilon/2 & \epsilon^2/4 \end{array} \right] +
\frac{1}{2}
 \left[ \begin{array}{cc}
1 & -\epsilon/2 \\
-\epsilon/2 & \epsilon^2/4 \end{array} \right] =
 \left[ \begin{array}{cc}
1 & 0 \\
0 & \epsilon^2/4 \end{array} \right].
\end{align*}
As a result, for both $x_1$ and $x_2$ we have
\begin{align*}
[1\;\; \epsilon/2] \Lambda_{\lambda}^{-1} 
\left[ \begin{array}{c}
1 \\
\epsilon/2 \end{array} \right]
=
[1\;\; \epsilon/2] 
 \left[ \begin{array}{cc}
1 & 0 \\
0 & 4/\epsilon^2 \end{array} \right]  
\left[ \begin{array}{c}
1 \\
\epsilon/2 \end{array} \right] = 2.
\end{align*}
On the other hand, if we consider the direction $y=x_1-x_2 = [0\;\; \epsilon]$, we have
\begin{align*}
[0\;\; \epsilon] \Lambda_{\lambda}^{-1} 
\left[ \begin{array}{c}
0 \\
\epsilon \end{array} \right]
=
[0\;\; \epsilon] 
 \left[ \begin{array}{cc}
1 & 0 \\
0 & 4/\epsilon^2 \end{array} \right]  
\left[ \begin{array}{c}
0 \\
\epsilon \end{array} \right] = 4.
\end{align*}
This example shows  that indeed the performance achieved by $\X\Y$ may be similar to the performance of $G$-optimal. Let us now consider a different setting where the two arms $x_1 = [1\;\; 0]$ and $x_2 = [1-\epsilon \;\; 0]$ are aligned on the same axis. In this case, the problem reduces to a 1-dimensional problem and both strategies would concentrate their allocation on $x_1 = [1\;\; 0]$ since it is the arm with larger norm and it may provide a better estimate of $\theta^*$. As a result, while the $G$-allocation has a performance of $1$, the $\X\Y$-allocation over the direction $[\epsilon \;\; 0]$ has a performance $\epsilon^2$, which can be arbitrarily smaller than $1$.

%%%%%%%%%%%%%%%%%%%%%%%%%%%%%%%%%%%%%%%%%%%%%%%%%%%%%%%%%%%%%%%%%%%%%%%%%%%%%%%
%% APPENDIX
%%%%%%%%%%%%%%%%%%%%%%%%%%%%%%%%%%%%%%%%%%%%%%%%%%%%%%%%%%%%%%%%%%%%%%%%%%%%%%%

%%%%%%%%%%%%%%%%%%%%%%%%%%%%%%%%%%%%%%%%%%%%%%%%%%%%%%%%%%%%%%%%%%%%%%%%%%%%%%%
%% PROOFS
%%%%%%%%%%%%%%%%%%%%%%%%%%%%%%%%%%%%%%%%%%%%%%%%%%%%%%%%%%%%%%%%%%%%%%%%%%%%%%%

\section{Proofs}\label{s:proofs}

\subsection{Lemmas}

\begin{proof}[Proof of Lemma~\ref{lem:oracle.stop}]
The proof follows from the fact that if $\S^*(\bx_n) \subseteq \C(x^*)$ and $\htheta_n \in \S^*(\bx_n)$ with high probability, then $\htheta_n\in\C(x^*)$ which implies that $\Pi(\htheta_n) = x^*$ by definition of the cone $\C(x^*)$.
\end{proof}

Before proceeding to the proof of Lemma~\ref{lem:complexity.range} we introduce the following technical tool.

%%%%%%%%%%%%%%%%%%%%%%%%%

\begin{proposition}[Equivalence-Theorem in~\citep{kiewol60}]\label{prop:equivalence.theorem}
Define $f(x;\xi) = x^\top M(\xi)^{-1} x$, where $M(\xi)$ is a $d\times d$ non-singular matrix and $x$ is a column vector in $\mathbb{R}^d$. We consider two extremum problems. 

The first is to choose $\xi$ so that
\begin{align*}
\hspace{-0.2cm}(1)~~ \xi ~~ \text{maximizes}~~\det M(\xi) \tag{\textit{D-optimal design}}
\end{align*}
The second one is to choose $\xi$ so that
\begin{align*}
(2)~~ \xi ~~ \text{minimizes}~~\max f(x;\xi) \tag{\textit{G-optimal design}} 
\end{align*}
We note that the integral with respect to $\xi$ of $f(x;\xi)$ is $d$; hence, $\max f(x;\xi) \geq d$, and thus a sufficient condition for $\xi$ to satisfy (2) is
\begin{align*}
\hspace{-3.5cm}(3)~~ \max f(x;\xi) = d. 
\end{align*}
Statements (1), (2) and (3) are equivalent.
\end{proposition}

\begin{proof}[Proof of Lemma~\ref{lem:complexity.range}]
%--------------------------------------------------
\textbf{Upper-bound.} We have the following sequence of inequalities
%--------------------------------------------------
%
\begin{align*}
\max_{y\in\Y*} \frac{||y||^2_{\Lambda_{\lambda}^{-1}}}{\Delta^2(y)} \leq \frac{1}{\Delta_{\min}^2}\max_{y\in\Y*} ||y||^2_{\Lambda^{-1}_{\lambda}} \leq \frac{4}{\Delta_{\min}^2}\max_{x\in\X} ||x||^2_{\Lambda^{-1}_{\lambda}},
\end{align*}
where the second inequality comes from a triangle inequality on $||y||^2_{\Lambda^{-1}_{\lambda}}$. Thus we obtain
%  ||y||^2 = ||x-x'||^2 <= ||x||^2 + 2 * ||x||^2 * ||x'||^2 + ||x'||^2 <= 4 * max ||x||^2. 
\begin{align*}
\rho^*(\lambda^*) =  \min_{\lambda\in\D^k}\max_{y\in\Y*} \frac{||y||^2_{\Lambda^{-1}_{\lambda}}}{\Delta^2(y)} \leq \frac{4}{\Delta_{\min}^2}\min_{\lambda\in\D^k} \max_{x\in\X} ||x||^2_{\Lambda^{-1}_{\lambda}} = \frac{4d}{\Delta_{\min}^2},
\end{align*}
where the last equality follows from the Kiefer-Wolfowitz equivalence theorem presented in Prop.~\ref{prop:equivalence.theorem}.

%\citep{pukelsheim2006optimal}.

%--------------------------------------------------
\textbf{Lower-bound.}
%--------------------------------------------------

%Let $y_{\min} = \arg\min_{y\in\Y^*} \Delta(y)$, then we have
%We have that
%%
%\begin{align*}
%\max_{y\in\Y*} \frac{||y||^2_{\Lambda_{\lambda}}}{\Delta^2(y)} \geq \frac{\max_{y\in\Y^*} ||y||^2_{\Lambda_{\lambda}}}{\Delta^2_{\min}}.
%\end{align*}
%%
We focus on the numerator $y^\top \Lambda_{\lambda}^{-1} y$. Since $\Lambda_{\lambda}$ is a positive definite matrix, we define its decomposition $\Lambda_{\lambda} = Q\Gamma Q^\top$, where $Q$ is an orthogonal matrix and $\Gamma$ is the diagonal matrix containing the eigenvalues. As a result the numerator can be written as
\begin{align*}
y^\top \Lambda_{\lambda}^{-1} y = y^\top Q \Gamma^{-1} Q^\top y = w^\top \Gamma^{-1} w,
\end{align*}
where we renamed $Q^\top y = w$. If we denote by $\gamma_{\max}$ the largest eigenvalue of $\Lambda_{\lambda}$ (i.e., the largest value in $\Gamma$), then
\begin{align*}
w^\top \Gamma^{-1} w \geq 1/\gamma_{\max} w^\top w = 1/\gamma_{\max} ||y||^2.
\end{align*}
The largest eigenvalue $\gamma_{\max}$ is upper-bounded by the sum of the largest eigenvalues of the matrices $\lambda(x) xx^\top$ which is $\lambda(x) ||x||_2$. As a result, we obtain the bound $\gamma_{\max} \leq \sum_{x}\lambda(x) ||x||_2 \leq L$, since $||x||_2\leq L$ and $\lambda$ is in the simplex. Thus we have
\begin{align*}
\min_{\lambda\in\D^k}\max_{y\in\Y*} \frac{||y||^2_{\Lambda_{\lambda}^{-1}}}{\Delta^2(y)} \geq \frac{1}{L}\max_{y\in\Y^*} \frac{||y||^2}{\Delta(y)^2} \geq \frac{\max_{y\in\Y^*}||y||^2}{L\Delta_{\min}^2}.
\end{align*}

%--------------------------------------------------
\textbf{Comparison with the $\mathbf{K}$-armed bandit complexity.}
%--------------------------------------------------

Finally, we show how the sample complexity reduces to the known quantity in the MAB case. 
If the arms in $\X$ coincide with the canonical basis of $\Re^d$, then for any allocation $\lambda$ the design matrix $\Lambda_{\lambda}$ becomes a diagonal matrix of the form $\text{diag}(\lambda(x_1), \ldots, \lambda(x_K))$. As a result, we obtain 
%\todoaout{Double-check this}
%
\begin{align*}
H_{\text{LB}} &= \min_{\lambda\in\D^k}\max_{y\in\Y^*}\frac{||y||_{\Lambda_{\lambda}^{-1}}^2}{\Delta^2(y)} = \min_{\lambda\in\D^k}\max_{x\in\X-\{x^*\}} \frac{1/\lambda(x) + 1/\lambda(x^*)}{\Delta^2(x)}.
\end{align*}
%
%which leads to an optimal allocation $\lambda(x) = 1/(\nu\Delta^2(x))$ and $\lambda(x^*) = 1/(\nu\Delta_{\min})$, with $\nu = 1/\Delta_{\min}^2 + \sum_{x\neq x^*} 1/\Delta^2(x)$, which corresponds to 
%

If we use the allocation $\lambda(x) = 1/(\nu\Delta^2(x))$ and $\lambda(x^*) = 1/(\nu\Delta_{\min})$, with $\nu = 1/\Delta_{\min}^2 + \sum_{x\neq x^*} 1/\Delta^2(x)$, we obtain 
%\begin{align*}
%H_{\text{LB}} = 2\nu = 2\Big(\frac{1}{\Delta_{\min}^2} + \sum_{x\neq x^*} \frac{1}{\Delta^2(x)}\Big) = 2H_{\text{MAB}}.
%\end{align*}
%
%------------------
\begin{align*}
H_{\text{LB}} & \leq \max_{x\in\X-\{x^*\}} \frac{\nu \Delta^2(x) + \nu \Delta_{\min}^2 }{\Delta^2(x)} = \max_{x\in\X-\{x^*\}} \nu + \nu \frac{\Delta_{\min}^2 }{\Delta^2(x)} \\
& = 2\nu = 2\Big(\frac{1}{\Delta_{\min}^2} + \sum_{x\neq x^*} \frac{1}{\Delta^2(x)}\Big) = 2H_{\text{MAB}}.
\end{align*}
%-----------------

On the other hand, letting $\tilde x$ be the second best arm and $\Delta(x^*) = \Delta_{\text{min}}$, we have that 
\begin{align*}
H_{\text{LB}} &= \min_{\lambda\in\D^k} \max_{x\neq x*} \frac{1/\lambda(x) + 1/\lambda(x^*)}{\Delta^2(x)}\\
& = \min_{\lambda\in\D^k} \max~ \bigg\{ \max_{x\neq x*} \frac{1/\lambda(x) + 1/\lambda(x^*)}{\Delta^2(x)} ; \frac{1/\lambda(\tilde x) + 1/\lambda(x^*)}{\Delta^2(x^*)}\bigg\} \\
& \geq \min_{\lambda\in\D^k} \max~ \bigg\{ \max_{x\neq x*} \frac{1/\lambda(x)}{\Delta^2(x)} ; \frac{1/\lambda(x^*)}{\Delta^2(x^*)}\bigg\} \\
& = \min_{\lambda\in\D^k} \max_{x \in \X}\frac{1/\lambda(x)}{\Delta^2(x)}.
\end{align*}

We set $\frac{1/\lambda(x)}{\Delta^2(x)}$ equal to a constant $c$ and thus we get $\lambda(x) = \frac{1}{c \Delta^2(x)}$. Since $\frac1c \sum_{x \in \X} \frac{1}{\Delta^2(x)} = 1$, it follows that:
\begin{align*}
c = \sum_{x \in \X} \frac{1}{\Delta^2(x)} = \sum_{x \neq x^*} \frac{1}{\Delta^2(x)} + \frac{1}{\Delta_\text{min}^2} = H_{\text{MAB}}.
\end{align*}

Thus, we get that $H_{\text{MAB}} \leq H_{\text{LB}} \leq 2 H_{\text{MAB}}$. This shows that $H_{\text{LB}}$ is a well defined notion of complexity for the linear best-arm identification problem and the corresponding  sample complexity $N^*$ is coherent with existing results in the MAB case.
\end{proof}

%%%%%%%%%%%%%%%%%%%%%%%%%

\begin{proof}[Proof of Lemma~\ref{lem:emp.stop}]
The proof follows from the fact that if $\hS(\bx_n) \subseteq \C(x)$ and $\theta^* \in \hS(\bx_n)$ with high probability, then $\theta^*\in\C(x)$ which implies that $\Pi(\htheta_n) = x = x^*$.
\end{proof}

%%%%%%%%%%%%%%%%%%%%%%%%%

\subsection{Proofs of Theorem~\ref{thm:g.allocation}~ and Theorem~\ref{thm:xy.allocation} }

\begin{proof}[Proof of Theorem~\ref{thm:g.allocation}]
The statement follows from Prop.~\ref{p:bound} and the performance guarantees for the different implementations of the $G$-optimal design. By recalling the empirical stopping condition in Eq.~\ref{eq:emp.stop.condition.2} and the definition $\rho^G(\lambda) = \max_{x} x^\top \Lambda_{\lambda}^{-1} x$, we notice that from a simple triangle inequality applied to $||y||_{A^{-1}}$, a sufficient condition for stopping is that for any $x\in\X$
\begin{align*}
\frac{4c^2 \rho_n^{\tilde G}\log_n(K^2/\delta)}{n} \leq \hDelta_n^2(x^*,x),
\end{align*}
where $\rho_n^{\tilde G} = \rho^G(\lambda_{\bx_n^{\tilde G}})$ and $\bx_n^{\tilde G}$ is the allocation obtained from rounding the optimal design $\lambda^G$ obtained from the continuous relaxation or the greedy incremental algorithm.
From Prop.~\ref{p:bound} we have that the following inequalities
\begin{align*}
\hDelta_n(x^*,x) \geq \Delta(x^*,x) - c||x^*-x||_{A_{\bx_n^G}^{-1}} \sqrt{\log_n(K^2/\delta)} \geq \Delta(x^*,x) - 2c  \sqrt{\frac{\rho_n^{\tilde G}\log_n(K^2/\delta)}{n}},
\end{align*}
hold with probability $1-\delta$. Combining this with the previous condition and since the condition must hold for all $x\in\X$, we have that a sufficient condition to stop using the $G$-allocation is
\begin{align*}
\frac{16c^2 \rho_n^{\tilde G}\log_n(K^2/\delta)}{n} \leq \Delta_{\min},% \enspace \Leftrightarrow \enspace \rho_n^{\tilde G} \leq \frac{\Delta_{\min}}{4c\sqrt{\log_n(K^2/\delta)}},
\end{align*}
which defines the level of accuracy that the $G$-allocation needs to achieve before stopping. Since $\rho_n^{\tilde G} \leq (1+\beta)d$ then the statement follows by inverting the previous inequality.
\end{proof}

\begin{proof}[Proof of Theorem~\ref{thm:xy.allocation}] 
We follow the same steps as in the proof of Theorem~\ref{thm:g.allocation}.
\end{proof}

%%%%%%%%%%%%%%%%%%%%%%%%%

\section{Implementation of the Allocation Strategies}\label{s:allocation.strategies}

In this section we discuss about possible implementations of the allocation strategies illustrated in sections~\ref{s:static} and~\ref{s:xy.adaptive} and we discuss their approximation accuracy guarantees.

\textbf{The efficient rounding procedure.}
We first report the general structure of the efficient rounding procedure defined in~\citep[Chapter~12]{pukelsheim2006optimal} to implement a design $\lambda$ into an allocation $\bx_n$ for any fixed number of steps $n$. Let $p=\text{supp}(\lambda)$ the support of $\lambda$,\footnote{For a fixed design $\lambda \in \Re^K$, we say that its \emph{support} is given by all arms in $\X$ whose corresponding features in $\lambda$ are different than 0.} then we want to compute the number of pulls $n_i$ (with $i=1,\ldots,p$) for all the arms in the support of $\lambda$. Basically, the fast implementation of the design is obtained in two phases, as follows:
\begin{itemize}[noitemsep,nolistsep] 
\item In the first phase, given the sample size $n$ and the number of support points $p$, we calculate their corresponding frequencies $n_i = \lceil{(n - \frac{1}{2}p)\lambda_i} \rceil$, where  $ n_1, n_2, \dots, n_p$ are positive integers with $ \sum_{i \leq p} n_i \geq n$. 
\item The second phase loops until the discrepancy $\big( \sum_{i \leq p} n_i \big) - n$ is $0$, either: 
\begin{itemize} 
\item increasing a frequency $n_j$ which attains $n_j/\lambda_j = \min_{i \leq p} (n-1)/\lambda_i$ to $n_{j+1}$, or 
\item decreasing some $n_k$ with $(n_k - 1)/\lambda_k = \max_{i \leq p} (n_i -1)/ \lambda_i$ to $n-1$.
\end{itemize}
\end{itemize}

An interesting feature of this procedure is that when moving from $n$ to $n+1$ the corresponding allocations $\bx_{n}$ and $\bx_{n+1}$ only differ for one element $i$ which is increased by 1, i.e., the discrete allocation is monotonic in $n$.

\textbf{Implementation of the $\mathbf{G}$-allocation.} %(see \citep{pukelsheim2006optimal} for a review). 
%this type of $NP$-hard discrete optimization problems .
A first option is to optimize a continuous relaxation of the problem and compute the optimal design. Let $\rho^G(\lambda) = \max_x x^\top \Lambda_{\lambda}^{-1} x$, the optimal design is
\begin{align}\label{eq:g.optimal.design.relax}
\lambda^G = \arg\min_{\lambda\in\D_k} \max_{x\in\X} ||x||^2_{\Lambda_{\lambda}^{-1}} = \arg\min_{\lambda\in\D_k}\rho^G(\lambda).
\end{align}
This is a convex optimization problem and it can be solved using the projected gradient algorithm, interior point techniques, or multiplicative algorithms. %!!!!!!!!(see e.g.,~\citep{yu2010monotonic}). 
To move from the design $\lambda^G$ to a discrete allocation we use the efficient rounding technique presented above and we obtain that the resulting allocation $\bx_t^{\tilde G}$ is guaranteed to be monotonic as the number of times an arm $x$ is pulled is non-decreasing with $t$. Thus from $\bx_t^{\tilde G}$ we obtain a simple incremental rule, where the arm $x_t$ is the arm for which $\bx_t^{\tilde G}$ recommends one pull more than in $\bx_{t-1}^{\tilde G}$. %Furthermore, the allocation $\bx_t^{\tilde G}$ is shown to have an efficiency loss of $p/n$ with respect to the ``ideal'' performance of the relaxed design $\lambda^G$ (where $p$ is the number of arms $x$ for which $\lambda^G(x) > 0$).\\
An alternative is to directly implement an incremental version of Eq.~\ref{eq:g.optimal.design} by selecting at each step $t$ the greedy arm
\begin{align}\label{eq:g.optimal.design.incremental}
x_t = \arg\min_{x\in\X} \max_{x'\in\X} x\!'^\top \!\big(A_{\bx_{t-1}}\!+ \!xx^\top\big)\!^{-1}\! x' 
= \arg\min_{x\in\X} \max_{x'\in\X} x\!'^\top\! \bigg[A^{-1}_{\bx_{t-1}}\! - \!\frac{A^{-1}_{\bx_{t-1}}xx^\top\! A^{-1}_{\bx_{t-1}}}{1+x^\top\! A^{-1}_{\bx_{t-1}}x}\bigg]\! x',
\end{align} 
where the second formulation follows from the matrix inversion lemma.
This allocation is somehow simpler and more direct than using the continuous relaxation but it may come with a higher efficiency loss.

%We report the performance guarantees for the two implementations proposed above. The allocation $\bx_t^{\tilde G}$ obtained applying the rounding procedure has the following performance guarantee.

%%%%%%%%%%%%%%%%%%%%%%%%%%

%Before proving Lemma~\ref{lem:xy.adaptive} we need an additional technical lemma. 
Before reporting the performance guarantees for the two implementations proposed above, we introduce an additional technical lemma which will be useful in the proofs on the performance guarantees. Although the lemma is presented for a specific definition of uncertainty $\rho$, any other notion including design matrices of the kind $\Lambda_{\lambda}$ will satisfy the same guarantee.
\begin{lemma}\label{lem:allocation.vs.design}
Let $\rho(\lambda) = \max_{x\in\X} x^{\top} \Lambda_{\lambda}^{-1} x$ be a measure of uncertainty of interest for any design $\lambda\in\D^K$. We denote by $\lambda^* = \arg\min_{\lambda\in\D^K} \rho(\lambda)$ the optimal design and for any $n > d$ we introduce the optimal discrete allocation as
\begin{align*}
\bx_n^* = \arg\min_{\bx_n \in \X^n} \max_{x\in\X} \frac{x^{\top} \Lambda_{\lambda_{\bx_n}}^{-1} x}{n},
\end{align*}
where $\lambda_{\bx_n}$ is the (fractional) design corresponding to $\bx_n$. Then we have
\begin{align}\label{eq:allocation.vs.design}
\rho(\lambda^*) \leq \rho(\bx_n^*) \leq \Big(1 + \frac{p}{n}\Big) \rho(\lambda^*),
\end{align}
where $p = \text{supp}(\lambda^*)$ is the number of points in the support of $\lambda^*$. If $d$ linearly independent arms are available in $\X$, then we can upper bound the size of the support of $\lambda^*$ and obtain
\begin{align}\label{eq:allocation.vs.design.2}
\rho(\lambda^*) \leq \rho(\bx_n^*) \leq \Big(1 + \frac{d(d+1)}{n}\Big) \rho(\lambda^*).
\end{align}
\end{lemma}

\begin{proof}
The first part of the statement follows by the definition of $\lambda^*$ as the minimizer of $\rho$. Let $\tilde \bx_n$ by an efficient rounding technique applied on $\lambda^*$ such as the one described in Lemma~12.8 in~\citep{pukelsheim2006optimal}. Then $\tilde \bx_n$ has the same support as $\lambda^*$ and an efficiency loss bounded by $p/n$. As a result, we have
\begin{align*}
\rho(\bx_n^*) \leq \rho(\tilde \bx_n) \leq \Big(1 + \frac{p}{n}\Big) \rho(\lambda^*),
\end{align*}
where the first inequality comes from the fact that $\bx_n^*$ is the minimizer of $\rho$ among allocations of length $n$. Then, from Caratheodory's theorem (see e.g.,~\citep{pukelsheim2006optimal} %\citep{fedorov_72}
) the number of support points used in $\lambda^*$ is upper bounded by $p\leq d(d+1)/2+1$ (under the assumption that there are $d$ linearly independent arms in $\X$). The final result follows by a rough maximization of $d(d+1)/2n + 1/n \leq d(d+1)/n$.
\end{proof}

\begin{remark}{1}
Note that the same upper-bound for the number of support points holds for any design, due to the properties of the design matrices. In fact, any design matrix is symmetric by construction, which implies that it is completely described by $D = d(d+1)/2$ elements and can thus be seen as a point in $\mathbb R^D$. Moreover, a design matrix  is a convex combination of a subset of points in $\mathbb R^D$ and thus it belongs to the convex hull of that subset of points. Caratheodory's theorem states that each point in the convex hull of any subset of points in $\mathbb R^D$ can be defined as a convex combination of at most $D+1$ points. It directly follows that any design matrix can be expressed using $(d(d+1)/2)+1$ points.
\end{remark}
%%%%%%%%%%%%%%%%%%%%%%%%%

It follows that the allocation $\bx_t^{\tilde G}$ obtained applying the rounding procedure has the following performance guarantee.

\begin{lemma}\label{lem:g.optimal.continuous.relax}
For any $t\geq d$, the rounding procedure defined in~\citep[Chapter~12]{pukelsheim2006optimal} returns an allocation $\bx_t^{\tilde G}$, whose corresponding design $\lambda^{\tilde G} = \lambda_{\bx_t^{\tilde G}}$ is such that\footnote{We recall that from any allocation $\bx_n$ the corresponding design $\lambda_{\bx}$ is such that $\lambda_{\bx_n}(x) = T_n(x)/n$.}
\begin{align*}
\rho^G(\lambda^{\tilde G}) \leq \Big(1+\frac{d+d^2+2}{2t}\Big) d.
\end{align*}
\end{lemma}

\begin{proof}[Proof of Lemma~\ref{lem:g.optimal.continuous.relax}]
We follow the same steps as in the proof of Lemma~\ref{lem:allocation.vs.design}~ to obtain the term $\beta = \frac{d+d^2+2}{2t}$. Then, noting that the performance of the optimal strategy $\rho^G(\lambda^{*G}) = d$ (from Prop.~\ref{prop:equivalence.theorem}), the results follows.
\end{proof}

% Submodularity issue -------------------------------------------------
\begin{comment}
On the other hand, the greedy incremental allocation achieves the following performance.
%
\begin{lemma}\label{lem:g.optimal.increment}
For any $t$, the greedy algorithm in Eq.~\ref{eq:g.optimal.design.incremental} returns an allocation $\bx_t^{\tilde G}$ whose corresponding design $\lambda^{\tilde G} = \lambda_{\bx_t^{\tilde G}}$ is such that 
%
\begin{align*}
\rho^G(\lambda^{\tilde G}) \leq (1+e^{-1})\Big(1+\frac{d+d^2+2}{2t}\Big) d.
\end{align*}
%
\end{lemma}

\begin{proof}[Proof of Lemma~\ref{lem:g.optimal.increment}]
The lemma directly follows from the remark that $G$-optimal design is equivalent to the $D$-optimal design which can be defined as the maximization of a submodular function. As a result, %!!!!!!!!!from Theorem~2.6 in~\citep{sagnol2013approximation} 
we obtain that the greedy algorithm achieves a constant fraction loss of $(1+e^{-1})$.
\end{proof}
\end{comment}
%--------------------------------------------------------------------

\textbf{Implementation of the \boldmath{$\X\Y$}-allocation.} Notice that the complexity of the $\X\Y$-allocation trivially follows from the complexity of the $G$-allocation and it is NP-hard. As a result, we need to propose approximate solutions to compute an allocation $\bx_n^{\widetilde{\X\Y}}$ as for the $G$-allocation. Let $\rho^{\X\Y}(\lambda) = \max_{y\in\Y} y^\top \Lambda_{\lambda}^{-1} y$, then the first option is the compute the optimal solution to the continuous relaxed problem
\begin{align}\label{eq:xy.optimal.design.relax}
\lambda^{\X\Y} = \arg\min_{\lambda\in\D_k} \max_{y\in\Y} ||y||^2_{\Lambda_{\lambda}^{-1}} = \arg\min_{\lambda\in\D_k}\rho^{\X\Y}(\lambda).
\end{align}
And then compute the corresponding discrete allocation $\bx_n^{\widetilde{\X\Y}}$ using the efficient rounding procedure. Alternatively, we can use an incremental greedy algorithm which at each step $t$ returns the arm
\begin{align}\label{eq:xy.optimal.design.incremental}
x_t = \arg\min_{x\in\X} \max_{y\in\Y} y^\top \big(A_{\bx_{t-1}} + xx^\top\big)^{-1} y.
\end{align} 

\begin{lemma}\label{lem:xy.optimal.continuous.relax}
For any $t\geq d$, the rounding procedure defined in~\citep[Chapter~12]{pukelsheim2006optimal} returns an allocation $\bx_t^{\widetilde {\X\Y}}$, whose corresponding design $\lambda^{\widetilde {\X\Y}} = \lambda_{\bx_t^{\widetilde {\X\Y}}}$ is such that
\begin{align*}
\rho^{\X\Y}(\lambda^{\widetilde {\X\Y}}) \leq 2\Big(1+\frac{d+d^2+2}{2t}\Big) d.
\end{align*}
\end{lemma}

\begin{proof}[Proof of Lemma~\ref{lem:xy.optimal.continuous.relax}]
The proof follows from the fact that for any pair $(x,x')$
$$||x-x'||_{A_{\bx_n}^{-1}} \leq 2\max_{x''\in\X} ||x''||_{A_{\bx_n}^{-1}}.$$
Then the proof proceeds as in Lemma~\ref{lem:g.optimal.continuous.relax}.
\end{proof}

% Submodularity issue -------------------------------------------------
\begin{comment}

On the other hand, the greedy incremental allocation achieves the following performance.
%
\begin{lemma}\label{lem:xy.optimal.increment}
For any $t$, the greedy algorithm in Eq.~\ref{eq:g.optimal.design.incremental} returns an allocation $\bx_t^{\tilde G}$ whose corresponding design $\lambda^{\tilde G} = \lambda_{\bx_t^{\tilde G}}$ is such that
%
\begin{align*}
\rho^G(\lambda^{\tilde G}) \leq (1+e^{-1})\Big(1+\frac{d+d^2+2}{2t}\Big) d.
\end{align*}
%
\end{lemma}

\begin{proof}[Proof of Lemma~\ref{lem:xy.optimal.increment}]
As for the proof of Lemma~\ref{lem:xy.optimal.continuous.relax}, after applying the inequality
$$||x-x'||_{A_{\bx_n}^{-1}} \leq 2\max_{x''\in\X} ||x''||_{A_{\bx_n}^{-1}},$$
we proceed as in Lemma~\ref{lem:g.optimal.increment}.
\end{proof}
\end{comment}
%----------------------------------------------------------------

\textbf{Implementation of \boldmath{$\X\Y$}-adaptive allocation.} The allocation rule in Eq.~\ref{eq:phase.allocation} basically coincides with the $\X\Y$-allocation and its properties extend smoothly.

%%%%%%%%%%%%%%%%%%%%%%%%%%%%%%%%%%

\section{Proof of Theorem \ref{thm:xy.adaptive}}

Before proceeding to the proof, we first report the proofs of two adittional lemmas.

%%%%%%%%%%%%%%%%%%%%%%%%%

\begin{proof}[Proof of Lemma~\ref{lem:dominated.arm}]
Let $y = x' - x$. Using the definition of $\hS(\bx_n)$ in Eq.~\ref{eq:emp.conf.set}, and the fact that $\theta^* \in \hS(\bx_n)$ with high probability, we have
\begin{align*}
(x' - x)^\top(\htheta_n - \theta^*) \leq c ||x' - x||_{A_{\bx}^{-1}}\sqrt{\log_n(K^2/\delta)}.
\end{align*}
Since the condition in Eq.~\ref{eq:discard.arm} is true, it follows that 
\begin{align*}
 (x' - x)^\top(\htheta_n - \theta^*) &\leq c ||x' - x||_{A_{\bx}^{-1}}\sqrt{\log_n(K^2/\delta)} \leq \hDelta_n(x',x) \Leftrightarrow\\
  - (x' - x)^\top \theta^* &\leq 0 \Leftrightarrow x^\top \theta^* \leq x{'^\top} \theta^*
\end{align*}
thus $x$ is dominated by $x'$ and $x$ cannot be the optimal arm.  
\end{proof}

%%%%%%%%%%%%%%%%%%%%%%%%

\begin{lemma}\label{lem:xy.adaptive}
For any phase $j$, the length is such that $n_j \leq \max\{M^*,\frac{16}{\alpha}N^*\}$ with probability $1-\delta$.
\end{lemma}

\begin{proof}[Proof of Lemma~\ref{lem:xy.adaptive}]
We first summarize the different quantities measuring the performance of an allocation strategy in different settings. For any design $\lambda\in\D^K$, we define
\begin{align}\label{eq:performance.def}
\rho^*(\lambda) = \max_{y\in\Y^*} \frac{||y||^2_{\Lambda_{\lambda}^{-1}}}{\Delta^2(y)}; \quad \rho^{\X\Y}(\lambda) = \max_{y\in\Y} ||y||^2_{\Lambda_{\lambda}^{-1}}; \quad \rho^j(\lambda) = \max_{y\in\hY_j} ||y||^2_{\Lambda_{\lambda}^{-1}}.
\end{align}
For any $n$, we also introduce the value of each of the previous quantities when the corresponding optimal (discrete) allocation is used
\begin{align}\label{eq:performance.opt}
\rho^*_n = \rho^*(\lambda_{\bx_n^*}); \quad \rho^{\X\Y}_n = \rho^{\X\Y}(\lambda_{\bx_n^{\X\Y}}); \quad \rho^j_n = \rho^j(\lambda_{\bx_n^j}).
\end{align}
Finally, we introduce the optimal designs
\begin{align}\label{eq:performance.opt.design}
\lambda^* = \arg\min_{\lambda\in\D^K} \rho^*(\lambda); \enspace \lambda^{\X\Y} = \arg\min_{\lambda\in\D^K} \rho^{\X\Y}(\lambda); \enspace \lambda^j = \arg\min_{\lambda\in\D^K} \rho^j(\lambda).
\end{align}

Let $\epsilon^*$ be the smallest $\epsilon$ such that there exists a pair $(x,x')$, with $x\neq x^*$ and $x'\neq x^*$, such that the confidence set $\S = \{\theta: \forall y\in\Y, |y^\top(\theta-\theta^*)| \leq \epsilon\}$ overlaps with the hyperplane $\C(x)\cap\C(x')$. Since $M^*$ is defined as the smallest number of steps needed by the $\X\Y$ strategy to avoid any overlap between $\S^*$ and the hyperplanes $\C(x)\cap\C(x')$, then we have that after $M^*$ steps
\begin{align}\label{eq:epsilon}
c\sqrt{\frac{\rho^{\X\Y}_{M^*}\log_n(K^2/\delta)}{M^*}} < \epsilon^*.
\end{align}
We consider two cases to study the length of a phase $j$.

\textbf{Case 1:} $\sqrt{\frac{\rho^j_{n_j}}{n_j}} \geq \frac{\epsilon^*}{c\sqrt{\log_n(K^2/\delta)}}$. From Eq.~\ref{eq:epsilon} it immediately follows that
\begin{align}\label{eq:case1}
\frac{\rho^j_{n_j}}{n_j} \geq \frac{\rho^{\X\Y}_{M^*}}{M^*}.
\end{align}
From definitions in Eqs.~\ref{eq:performance.def} and~\ref{eq:performance.opt}, since $\hY_j \subseteq \Y$ we have for any $n$, $\rho^j_n \leq \rho^{\X\Y}_n$. As a result, if $n_j \geq M^*$, since $\rho^j_n/n$ is a non-increasing function, then we would have the sequence of inequalities
\begin{align*}
\frac{\rho^j_{n_j}}{n_j} \leq \frac{\rho^j_{M^*}}{M^*} \leq \frac{\rho^{\X\Y}_{M^*}}{M^*},
\end{align*}
which contradicts Eq.~\ref{eq:case1}. Thus $n_j \leq M^*$.

\textbf{Case 2:} $\sqrt{\frac{\rho^j_{n_j}}{n_j}} \leq \frac{\epsilon^*}{c\sqrt{\log_n(K^2/\delta)}}$. We first relate the performance at phase $j$ with the performance of the oracle. For any $n$
\begin{align*}
\rho^j_n &= \rho^j(\lambda_{\bx_n^j}) \leq \rho^j(\lambda_{\bx_n^*}) = \max_{y\in\hY_j} y^\top \Lambda_{\lambda_{\bx_n^*}}^{-1} y = \max_{y\in\hY_j} \frac{y^\top \Lambda_{\lambda_{\bx_n^*}}^{-1} y}{\Delta^2(y)} \Delta(y) \leq \max_{y\in\hY_j} \frac{y^\top \Lambda_{\lambda_{\bx_n^*}}^{-1} y}{\Delta^2(y)} \max_{y\in\hY_j} \Delta^2(y).
\end{align*}
If now we consider $n=n_j$, then the definition case 2 implies that the estimation error $\sqrt{\rho^j_{n_j}/n_j}$ is small enough so that all the directions in $\Y-\Y^*$ have already been discarded from $\hY_j$ and $\hY_j \subseteq\Y^*$. Thus
\begin{align}\label{eq:phase.to.oracle}
\rho^j_{n_j} \leq \max_{y\in \Y^*} \frac{y^\top \Lambda_{\lambda_{\bx_{n_j}^*}}^{-1} y}{\Delta^2(y)} \max_{y\in\hY_j} \Delta^2(y) = \rho^*_{n_j}\max_{y\in\hY_j} \Delta^2(y).
\end{align}
This relationship does not provide a bound on $n_j$ yet. We first need to recall from Prop.~\ref{p:bound} that for any $y\in\Y$ (and notably for the directions in $\hY_j$) we have
\begin{align*}
|y^\top (\htheta_{j-1} - \theta^*)| \leq c\sqrt{y^\top A^{-1}_{j-1}y \log_n(K^2/\delta)}, 
\end{align*}
where $A_{j-1} = A_{\bx_{n_{j-1}}^{j-1}}$ is the matrix constructed from the pulls within phase $j-1$. Since $\bx_{n}^{j-1}$ is obtained from a $\X\Y$-allocation applied on directions in $\hY_{j-1}$, we obtain that for any $y\in\hY_j$
\begin{align*}
|y^\top (\htheta_{j-1} - \theta^*)| \leq c\sqrt{ \log_n(K^2/\delta)}\max_{y\in\hY_{j-1}}\sqrt{y^\top A^{-1}_{j-1}y} = c\sqrt{ \frac{\log_n(K^2/\delta)\rho^{j-1}_{n_{j-1}}}{n_{j-1}}},  
\end{align*}
Reordering the terms in the previous expression we have that for any $y\in\hY_j$
\begin{align*}
\Delta(y) \leq \hDelta_{j-1}(y) + c\sqrt{ \frac{\log_n(K^2/\delta)\rho^{j-1}_{n_{j-1}}}{n_{j-1}}}.
\end{align*}
Since the direction $y$ is included in $\hY_{j}$ then the discard condition in Eq.~\ref{eq:discard.arm} failed for $y$, implying that $\hDelta_{j-1}(y) \leq c\sqrt{ \frac{\log_n(K^2/\delta)\rho^{j-1}_{n_{j-1}}}{n_{j-1}}}$. Thus we finally obtain
\begin{align*}
\max_{y\in\hY_j} \Delta(y) \leq 2c\sqrt{ \frac{\log_n(K^2/\delta)\rho^{j-1}_{n_{j-1}}}{n_{j-1}}}.
\end{align*}
Combining this with Eq.~\ref{eq:phase.to.oracle} we have
\begin{align*}
\rho^j_{n_j} \leq \rho^*_{n_j} 4c^2\frac{\log_n(K^2/\delta)\rho^{j-1}_{n_{j-1}}}{n_{j-1}}.
\end{align*}
Using the stopping condition of phase $j$ and the relationship between the performance $\rho^j$, we obtain that at time $\bar n=n_j-1$
\begin{align*}
\frac{\rho^j_{\bar n}}{\bar n} \geq \alpha \frac{\rho^{j-1}_{n_{j-1}}}{n_{j-1}} \geq \frac{\alpha}{4c^2\log_n(K^2/\delta)} \frac{\rho^j_{n_j} }{\rho^*_{n_j}}
\end{align*}
We can further refine the previous inequality as
\begin{align*}
\frac{\rho^j_{\bar n}}{\bar n} \geq \frac{\alpha \rho^*_{N^*}}{4 N^*} \frac{N^*}{c^2 \log_n(K^2/\delta)\rho^*_{N^*}} \frac{\rho^j_{n_j} }{\rho^*_{n_j}} \geq \frac{\alpha \rho^*_{N^*}}{4 N^*} \frac{\rho^j_{n_j} }{\rho^*_{n_j}},
\end{align*}
where we use the definition of $N^*$ in Eq.~\ref{eq:oracle.complexity}, which implies $c\sqrt{ \log_n(K^2/\delta)\rho^*_{N^*}/N^*} \leq 1$. Reordering the terms and using $\bar n = n_j-1$, we obtain
\begin{align*}
n_j \leq 1 + \frac{4 N^*}{\alpha} \frac{\rho_{n_j-1}^j}{\rho^j_{n_j}} \frac{\rho^*_{n_j}}{\rho^*_{ N^*}}.
\end{align*}
From Lemma~\ref{lem:allocation.vs.design} and the optimal designs defined in Eq.~\ref{eq:performance.opt.design} we have
\begin{align*}
n_j \leq 1 + \frac{4  N^*}{\alpha} \frac{(1+d(d+1)/(n_j-1))\rho^j(\lambda^j)}{\rho^j(\lambda^j)} \frac{(1+d(d+1)/(n_j-1))\rho^*(\lambda^*)}{\rho^*(\lambda^*)}.
\end{align*}
Using the fact that the algorithm forces $n_j \geq d(d+1) + 1$, the statement follows.
\end{proof}

%%%%%%%%%%%%%%%%%%%%%%%%%

\begin{proof}[Proof of Theorem~\ref{thm:xy.adaptive}]
Let $J$ be the index of any phase for which $|\hX_{J}|>1$. Then there exist at least one arm $x\in\X$ (beside $x^*$) for which the discarding condition in Lemma~\ref{lem:dominated.arm} is not triggered, which corresponds to the fact that for all arms $x'\in\X$
\begin{align*}
c ||x-x'||_{A_{\bx_{n_J}^J}^{-1}} \sqrt{\log_n(K^2/\delta)} \geq \hDelta_J(x,x').
\end{align*}
By developing the right hand side, we have
\begin{align*}
\hDelta_J(x,x') \geq \Delta(x,x') - c ||x-x'||_{A_{\bx_{n_J}^J}^{-1}} \sqrt{\log_n(K^2/\delta)} \geq \Delta_{\min} - c \sqrt{\frac{\rho_{n_J}^J\log_n(K^2/\delta)}{n_J}}
\end{align*}
which leads to the condition
\begin{align}\label{eq:no.stop}
2c \sqrt{\frac{\rho_{n_J}^J\log_n(K^2/\delta)}{n_J}} \geq \Delta_{\min}.
\end{align}
Using the phase stopping condition and the initial value of $\rho^0$ we have
\begin{align*}
\frac{\rho_{n_J}^J}{n_J} \leq \alpha \frac{\rho_{n_{J-1}}^{J-1}}{n_{J-1}} \leq \alpha^J \frac{\rho^0}{n_0} = {\alpha^J}.
\end{align*}
By joining this inequality with Eq.~\ref{eq:no.stop} we obtain
\begin{align*}
\alpha^J \geq \frac{\Delta^2_{\min}}{4c^2 \log_n(K^2/\delta)},
\end{align*}
and it follows that $J \leq \log(4c^2 \log_n(K^2/\delta) / \Delta^2_{\min}) / \log(1/\alpha)$ which together with Lemma~\ref{lem:xy.adaptive} leads to the final statement.
\end{proof}

%%%%%%%%%%%%%%%%%%%%%%%%%%%%%%%%%%%%%%%%%%%%%%%%%%%%%%%%%%%%%%%%%%%%%%%%%%%%%%%
%% Additional Empirical Results
%%%%%%%%%%%%%%%%%%%%%%%%%%%%%%%%%%%%%%%%%%%%%%%%%%%%%%%%%%%%%%%%%%%%%%%%%%%%%%%

\section{Additional Empirical Results}\label{s:emp.result.plus}

For the setting described in Sec.~\ref{s:experiments}, in order to point out the different repartitions of the sampling budget over arms, in Fig.~\ref{f:avg_pulls} we present the number of samples allocated per arm, for the case when the input space $\X \subseteq \mathbb{R}^5$. We remind that the arms denoted $x_1, \dots, x_5$ form the canonical basis and arm $x_6 = [\cos(\omega)~~\sin(\omega)~~0~~0~~0]$. 

%
%\begin{wrapfigure}[12]{rh}{0.6\textwidth}
\begin{figure}[h]%[12]{rh}{0.6\textwidth}
\centering
\resizebox{0.8\textwidth}{!}{%
\begin{tabular}{|l|r|r|r|r|r|}
\hline
\textbf{Samples/arm} & \emph{$\X\Y$-oracle}  & \emph{$\X\Y$-adaptive}   & \emph{$\X\Y$ }   & \emph{$G$}      & \emph{Fully-adaptive} \\ \hline
$x_1$              & 207            & 263               &  29523    & 28014    & 740    \\ \hline
$x_2$              & 41440          & 52713             &  29524    & 28015    & 149220 \\ \hline
$x_3$              & 2              & 3                 &  29524    & 28015    & 1      \\ \hline
$x_4$              & 2              & 5                 &  29524    & 28015    & 1      \\ \hline
$x_5$              & 1              & 2                 &  29524    & 28015    & 1      \\ \hline
$x_6$              & 0              & 2                 &  1        & 1        & 1      \\ \hline
\textbf{Budget}    & 41652          & 52988             &  147620   & 140075   & 149964 \\ \hline
\end{tabular}
}
\caption{\small{The budget needed by the allocation strategies to identify the best arm when $\X \subseteq \mathbb{R}^5$ and their sample allocation over arms. $\mathcal{XY}$ and $G$ allocate samples uniformly over the canonical arms while $\mathcal{XY}$-oracle and $\mathcal{XY}$-adaptive use most of the samples for arm $x_2$ (corresponding to the most informative direction).}} 
\label{f:avg_pulls} 
\end{figure}
%\end{wrapfigure}

We can notice that even though the Fully-adaptive algorithm identifies the most informative direction and focuses the sampling on arm $x_2$, its sample complexity still has a growth linear in the dimension, due to the extra $\sqrt{d}$ term in his bound. Consequently, the advantage over the static strategies is canceled.  On the other hand, $\X\Y$-adaptive ``learns'' the gaps  from the observations and allocates the samples very similarly to $\X\Y$-oracle, without suffering a large loss in terms of the sampling budget. However, $\X\Y$-adaptive's sample complexity has to account for the the re-initializations made at the beginning of a new phase. 

Finally, we notice that in this problem that static allocations, $\X\Y$ and $G$, perform a uniform allocation over the canonical arms. Another interesting remark is that the number of pulls to one canonical arm is smaller than the samples that $\X\Y$-oracle allocated to $x_2$. This is explained by the ``mutual information'' coming from the multiple observations on all directions, which helps in reducing the overall uncertainty of the confidence set.

\end{document}